\documentclass{article} 
\usepackage{matharticle}

\title{Robust Sparse Estimation Tasks in High Dimensions}
\usepackage{enumerate}
\usepackage{algpseudocode}
\usepackage{algorithm}
\usepackage{times}
\usepackage{amsfonts}
\usepackage{sectsty}
\usepackage{xypic}

 \author{Jerry Li\thanks{Supported by NSF grant CCF-1217921, DOE grant DE-SC0008923, NSF CAREER Award CCF-145326, and a NSF Graduate Research Fellowship}\\EECS, MIT\\\tt{jerryzli@mit.edu}}

\renewcommand{\tr}{\mathrm{tr}}
\newcommand{\eps}{\varepsilon}

\newcommand{\Sigmahat}{\widehat{\Sigma}}
\newcommand{\supp}{\mathrm{supp}}
\newcommand{\X}{\mathcal{X}}
\newcommand{\W}{\mathcal{W}}

\newcommand{\Xt}{\mathcal{W}^{(2)}}
\newcommand{\U}{\mathcal{U}}
\newcommand{\Y}{\mathcal{Y}}
\newcommand{\Sgood}{S_{\mathrm{good}}}
\newcommand{\Sbad}{S_{\mathrm{bad}}}
\newcommand\numberthis{\addtocounter{equation}{1}\tag{\theequation}}
\newcommand{\muhat}{\widehat{\mu}}
\newcommand{\mutilde}{\widetilde{\mu}}

\newcommand{\wg}{w^g}
\newcommand{\wb}{w^b}
\newcommand{\Otilde}{\widetilde{O}}
\newcommand{\Omegatilde}{\widetilde{\Omega}}
\newcommand{\dtv}{d_{\mathrm{TV}}}

\begin{document}
\maketitle

\begin{abstract}
In this paper we initiate the study of whether or not sparse estimation tasks can be performed efficiently in high dimensions, in the robust setting where an $\eps$-fraction of samples are corrupted adversarially.
We study the natural robust version of two classical sparse estimation problems, namely, sparse mean estimation and sparse PCA in the spiked covariance model.
For both of these problems, we provide the first efficient algorithms that provide non-trivial error guarantees in the presence of noise, using only a number of samples which is similar to the number required for these problems without noise.
In particular, our sample complexities are sublinear in the ambient dimension $d$.
Our work also suggests evidence for new computational-vs-statistical gaps for these problems (similar to those for sparse PCA without noise) which only arise in the presence of noise.
\end{abstract}

\section{Introduction}

In the last couple of decades, there has been a large amount of work in machine learning and statistics on how to exploit sparsity in high dimensional data analysis.
Motivated by the ever-increasing quantity and dimensionality of data, the goal at a high level is to utilize the underlying sparsity of natural data to extract meaningful guarantees using a number of samples that is sublinear in the dimensionality of the data.
In this paper, we will consider the unsupervised setting, where we have sample access to some distribution with some underlying sparsity, and our goal is to recover this distribution by exploiting this structure.
Two natural and well-studied problems in this setting that attempt to exploit sparsity are sparse mean estimation and sparse PCA.
In both problems, the shared theme is that we assume that one wishes to find a distinguished sparse direction of a Gaussian data set.
However, the algorithms inspired by this line of work tend to be quite brittle---it can be shown that they fail when the model is slightly perturbed.

This connects to a major concern in high dimensional data analysis: that of \emph{model misspecification}.
At a high level, the worry is that our algorithms should be able to tolerate the case when our assumed model and the true model do not perfectly coincide.
In the distributional setting, this (more or less) corresponds to the regime when a small fraction of our samples are \emph{adversarially} corrupted.
The study of these so-called \emph{robust estimators}, i.e., estimators which work in the presence of such noise, is a classical subfield of statistics.
Unfortunately, the classical algorithms for these problems fail to scale as the dimensionality of the problem grows---either the algorithms run in time which is exponential in the dimension, or the error guarantees for these algorithms degrade substantially as the dimension grows.
In a flurry of recent work, we now know new algorithms which circumvent this ``curse of dimensionality'': they run efficiently, and provide dimension independent error guarantees.
However, these algorithms are unable to exploit any inherent sparsity in the problem.

This raises the natural ``meta-question'': 
\begin{question}
Do the statistical gains (achievable by computationally efficient algorithms) for sparse estimation problems persist in the presence of noise?
\end{question}

More formally: Suppose we are asked to solve some estimation task given samples from some distribution $D$ with some underlying sparsity constraint (e.g. sparse PCA).
Suppose now an $\eps$-fraction of the samples are corrupted.
Can we still solve the same sparse estimation problem?
In this work, we initiate the study of such issues.
Interestingly, new gaps between computational and statistical rates seem to emerge in the presence of noise.
In particular, while the sparse mean estimation problem was previously quite simple to solve, the efficient algorithms which achieve the minimax rate for this problem break down in the presence of this adversarial noise.
More concretely, it seems that the efficient algorithms which are robust to noise run into the same computational issues as those which plague sparse PCA.
A very interesting question is whether this phenomenon is inherent to any computationally efficient algorithm.

\subsection{Our Contribution}
We study the natural robust versions of two classical, well-studied statistical tasks involving sparsity, namely, sparse mean estimation, and sparse PCA.

\noindent
\paragraph{Robust sparse mean estimation} Here, we get a set of $d$-dimensional samples from $\normal (\mu, I)$, where $\mu$ is $k$-sparse, and an $\eps$-fraction of the points are corrupted adversarially.
Our goal then is to recover $\mu$.
Our main contribution is the following:
\begin{theorem}[informal, see Theorem \ref{thm:spmean}]
There is an efficient algorithm, which given a set of $\eps$-corrupted samples of size $\Otilde(\frac{k^2 \log d}{\eps^2})$ from $\normal (\mu, I)$ where $\mu$ is $k$-sparse, outputs a $\muhat$ so that with high probability, $\| \muhat - \mu \|_2 \leq \eps \sqrt{\log 1 / \eps}$.
\end{theorem}

The recovery guarantee we achieve, namely $O( \eps \sqrt{\log 1 / \eps})$, is off by the optimal guarantee by only a factor of $\sqrt{\log 1 / \eps}$.
Moreover, results of \cite{DKS16c} imply that our bound is tight for any efficient SQ algorithm.
One can show that information theoretically, it suffices to take $O(\frac{k \log d}{\eps^2})$ samples to learn the mean to $\ell_2$ error $O(\eps)$, even with corrupted data.
Without model misspecification, this problem is quite simple algorithmically: it turns out that the truncated empirical mean achieves the information theoretically optimal rate.
However, efficient algorithms  for this task break down badly given noise, and to our knowledge there is no simple way of fixing them.
Very interestingly, the rate we achieve is off from this information theoretic rate by a $k^2$ vs $k$ factor---the same computational vs. statistical gap that arises in sparse PCA.
This phenomenon only seems to appear in the presence of noise, and we conjecture that this is inherent:

\begin{conjecture}
Any efficient algorithm for robust sparse mean estimation needs $\Omegatilde(\frac{k^2 \log d}{\eps^2})$ samples.
\end{conjecture}

In Appendix \ref{app:barrier} we give some intuition for why it seems to be true.
At a high level, it seems that any technique to detect outliers for the mean must look for sparse directions in which the variance is much larger than it should be; at which point the problem faces the same computational difficulties as sparse PCA.
We leave closing this gap as an interesting open problem.

\noindent
\paragraph{Robust sparse PCA} Here, we study the natural robust analogue of the spiked covariance model.
Classically, two problems are studied in this setting.
The \emph{detection} problem is given as follows: given sample access to the distributions, we are asked to distinguish between $\normal (0, I)$, and $\normal (0, I + \rho vv^T)$ where $v$ is a $k$-sparse unit vector.
That is, we wish to understand if we can detect the presence of any sparse principal component.
Our main result is the following:

\begin{theorem}[informal, see Theorem \ref{thm:rspca1}]
Fix $\rho > 0$, and let $\eta = O(\eps \sqrt{\log 1 / \eps})$.
If $\rho > \eta$, there is an efficient algorithm, which given a set of $\eps$-corrupted samples of size $O(\frac{k^2 \log d}{\rho^2})$ which distinguishes between  $\normal (0, I)$, and $\normal (0, I + \rho vv^T)$ with high probability.
\end{theorem}

The condition that $\eps = \Otilde (\rho)$ is necessary (up to log factors), as otherwise the problem is impossible information theoretically.
Observe that this (up to log factors) matches the optimal rate for computationally efficient detection for sparse PCA without noise (under reasonable complexity theoretic assumptions, see \cite{BR13,WBS16}), and so it seems that noise does not introduce an additional gap here.
The \emph{recovery} problem is similar, except now we want to recover the planted spike $v$, i.e. find a $u$ minimizing $L(u, v) = \frac{1}{\sqrt{2}} \| uu^T - vv^T \| $, which turns out to be the natural measure for this problem.
For this, we show:
\begin{theorem}[informal, see Theorem \ref{thm:rspca2}]
Fix $\eps> 0$ and $0 < \rho = O(1)$, and let $\eta = O(\eps \sqrt{\log 1 / \eps})$.
There is an efficient algorithm, which given a set of $\eps$-corrupted samples of size $O(\frac{k^2 \log d}{\eta^2})$ from $\normal (0, I + \rho vv^T)$, outputs a $u$ so that
$
L(u, v) = O \left( \frac{\eta}{\rho} \right) 
$
with high probability.
\end{theorem}

This rate is non-trivial---in particular, it provides guarantees for recovery of $v$ when the number of samples we take is at the detection threshold.
Moreover, up to log factors, our rate is optimal for computationally efficient algorithms--\cite{WBS16} gives an algorithm with rate roughly $O(\eps / \rho)$, and show that this is necessary.

\paragraph{Techniques}
We first introduce a simple way to describe the optimization problems used for solving sparse mean estimation and sparse PCA.
This approach is very similar to the approach taken by ~\cite{CRPW12} for solving under-determined linear systems.
We observe that any set $\mathcal{S}$ in a Hilbert space naturally induces a dual norm
$
\| x \|^*_\mathcal{S} = \max_{y \in \mathcal{S}} |\langle x, y \rangle| ,
$
and that well-known efficient algorithms for sparse mean estimation and sparse PCA simply compute this norm, and the corresponding dual witness $y \in \mathcal{S}$ which maximizes this norm, for appropriate choices of $\mathcal{S}$.
These norms give us a language to only consider deviations in directions we care about, which allows us to prove concentration bounds which are not true for more traditional norms.

We now describe our techniques for robust sparse mean estimation.
Our starting point is the convex programming approach of ~\cite{DKKLMS16}.
We assign each sample point a weight, which morally corresponds to our belief about whether the point is corrupted, and we optimize these weights.
In previous work of \cite{DKKLMS16}, the approach was to find weights so that the empirical covariance with these weights looked like the identity in spectral norm.

Unfortunately, such an approach fundamentally fails for us because the spectrum of the covariance will never concentrate for us with the number of samples we take.
Instead, we utilize a novel connection to sparse PCA.
We show that if instead we find weights so that the empirical covariance with these weights looks like the identity in the dual norm induced by a natural SDP for sparse PCA (in the noiseless setting), then this suffices to show that the trucnated empirical mean with these weights is close to the truth.
We do so by convex programming.
While we cannot explicitly write down the feasible set of weights, it is a convex set.
Thus, by the classical theory of convex optimization, it suffices to give a separation oracle for this convex set to optimize over this set.
We show that in fact the SDP for sparse PCA gives us such a separation oracle, if one is sufficiently careful to always work with sparsity preserving objects.
This in turns suffices to allow us to (approximately) find a point in the desired feasible set of points, which we show suffices to recover the true mean.

We now turn to robust sparse PCA.
We first consider the detection problem, which is somewhat easier technically.
Here, we again use the dual norm induced by the SDP for sparse PCA.
We show that if we can find weights on the samples (as before) so that the empirical covariance with these samples has minimal dual norm, then the value of the dual norm gives us a distinguisher between the spiked and non-spiked case.
To find such a set of weights, we observe that norms are convex, and thus our objective is convex.
Thus, as before, to optimize over this set it suffices to give a separation oracle, which again the SDP for sparse PCA allows us to do.

We now turn our attention to the recovery problem.
Here, the setup is very similar, except now we simultaneously find a set of weights and an ``explainer'' matrix $A$ so that the empirical covariance with these weights is ``maximally explained'' by $A$, in a norm very similar to the one induced by the sparse PCA SDP.
Utilizing that norms are convex, we show that this can be done via a convex program using the types of techniques described above, and that the top eigenvector of the optimal $A$ gives us the desired solution.
While the convex program would be quite difficult to write down in one shot, it is quite easily expressible using the abstraction of dual norms.

\subsection{Related Work}

As mentioned previously, there has been a large amount of work on various ways to exploit sparsity for machine learning and statistics.
In the supervised setting, perhaps the most well-known of these is compressive sensing and its variants (see \cite{CW08,HTW15} for more details).
We do not attempt to provide an exhaustive overview the field here.
Other well-known problems in the same vein include general classes of linear inverse problems, see \cite{CRPW12} and matrix completion (\cite{CR12}).

The question of estimating a sparse mean is very related to a classical statistical model known as the \emph{Gaussian sequence model}, and the reader is referred to \cite{Tsy09,Joh11,Rig15} for in-depth surveys on the area. 
This problem has also garnered a lot of attention recently in various distributed and memory-limited settings, see \cite{GMN14,SD15,BGMNW16}.
The study of sparse PCA was initiated in \cite{Joh01} and since yielded a very rich algorithmic and statistical theory (\cite{DE-GJL07, DBG08, AW08, WTH09, JNRS10, ACD11, LZ12, Ma13, BJNP13, CMW13, OMH14,  GWL14, CRZ16, BMVX16, PWBM16}).
In particular, we highlight a very interesting line of work \cite{BR13,KNV15,MW15,WGL15,WBS16}, which give evidence that any computationally efficient estimator for sparse PCA must suffer a sub-optimal statistical rate rate.
We conjecture that a similar phenomenon occurs when we inject noise into the sparse mean estimation problem.

In this paper we consider the classical notion of corruption studied in robust statistics, introduced back in the 70's in seminal works of \cite{HR09,Tuk75, HRRS86}.
Unfortunately, essentially all robust estimators require exponential time in the dimension to compute (\cite{JP78,Ber06,HM13}).
Subsequent work of \cite{LT15,BS15} gave efficient SDP-based estimators for these problems which unfortunately had error guarantees which degraded polynomially with the dimension.
However, a recent flurry of work (\cite{DKKLMS16,LRV16,CSV16,DKKLMS17,DKS17, DKS16c}) have given new, computationally efficient, robust estimators for these problems and other settings which avoid this loss, and are often almost optimal.
Independent work of \cite{DBS17} also considers the robust sparse setting. 
They give a similar result for robust mean estimation, and also consider robust sparse PCA, though in a somewhat different setting than we do, as well as robust sparse linear regression.

The questions we consider are similar to learning in the presence of malicious error studied in \cite{Val85,KL93}, which has received a lot of attention, particularly in the setting of learning halfspaces (\cite{Ser03,KLS09,ABL14}).
They also are connected to work on related models of robust PCA (\cite{Bru09,CLMW11,LMTZ12,ZL14}).
We refer the reader to \cite{DKKLMS16} to a detailed discussion on the relationships between these questions and the ones we study.

\section{Definitions}
Throughout this paper, if $v$ is a vector, we will let $\| v \|_2$ denote its $\ell_2$ norm.
If $M$ is a matrix, we let $\| M \|$ denote its spectral norm, we let $\| M \|_F$ denote its Frobenius norm, and we let $\| M \|_1 = \sum_{ij} |M_{ij}|$ be its $\ell_1$-norm if it were considered a vector.
For any two distributions $F, G$ over $\R^d$, we let $\dtv(F, G) = \frac{1}{2} \int_{\R^d} |F - G| d x$ denote the total variation distance between the two distributions.

We will study the following contamination model:
\begin{definition}[$\eps$-corruption]
We say a a set of samples $X_1, X_2, \ldots, X_n$ is an $\eps$-corrupted set of samples from a distribution $D$
if it is generated by the process following process. First, we draw $n$ independent samples from $D$. Then, an adversary inspects these samples, and changes an $\eps$-fraction of them \emph{arbitrarily}, then returns these new points to us, in any order.
Given an $\eps$-corrupted set of samples, we let $\Sgood \subseteq [n]$ denote the indices of the uncorrupted samples, and we let $\Sbad \subseteq [n]$ denote the indices of the corrupted samples.
\end{definition}

As discussed in \cite{DKKLMS16}, this is a strong notion of sample corruption that is able to simulate previously defined notions of error.
In particular, this can simulate (up to constant factors) the scenario when our samples do not come from $D$, but come from a distribution $D'$ with total variation distance at most $O(\eps)$ from $D$.

We may now formally define the algorithmic problems we consider.

\paragraph{Robust sparse mean estimation}
Here, we assume we get an $\eps$-corrupted set of samples from $\normal (\mu, I)$, where $\mu$ is $k$-sparse. 
Our goal is to recover $\mu$ in $\ell_2$.
It is not hard to show that there is an exponential time estimator which achieves rate $\Otilde(k \log d / \eps^2)$, and moreover, this rate is optimal (see Appendix \ref{app:info}).
However, this algorithm requires highly exponential time.
We show:

\begin{theorem}[Efficient robust sparse mean estimation]
\label{thm:spmean}
Fix $\eps, \delta > 0$, and let $k$ be fixed.
Let $\eta = O(\eps \sqrt{\log 1 / \eps})$.
Given an $\eps$-corrupted set of samples $X_1, \ldots, X_n \in \R^d$ from $\normal (\mu, I)$, where $\mu$ is $k$-sparse, and 
\[
n = \Omega \left( \frac{\min (k^2, d) + \log \binom{d^2}{k^2} + \log 1 / \delta}{\eta^2} \right) \; ,
\]
there is a poly-time algorithm which outputs $\muhat$ so that w.p. $1 - \delta$, we have $\| \mu - \muhat \|_2 \leq O(\eta)$.
\end{theorem}
It is well-known that information theoretically, the best error one can achieve is $\Theta(\eps)$, as achieved by Fact \ref{fact:spMeanInef}.
We show that it is possible to efficiently match this bound, up to a $\sqrt{\log 1 / \eps}$ factor.
Interestingly, our rate differs from that in Fact \ref{fact:spMeanInef}: our sample complexity is (roughly) $\Otilde(k^2 \log d / \eps^2)$ versus $O(k \log d / \eps^2)$.
We conjecture this is necessary for any efficient algorithm.

\paragraph{Robust sparse PCA}
We will consider both the detection and recovery problems for sparse PCA.
We first focus \emph{detection problem} for sparse PCA.
Here, we are given $\rho > 0$, and an $\eps$-corrupted set of samples from a $d$-dimensional distribution $D$, where $D$ can is either $\normal (0, I)$ or $\normal (0, I + \rho vv^T)$ for some $k$-sparse unit vector $v$.
Our goal is to distinguish between the two cases, using as few samples as possible.
It is not hard to show that information theoretically, $O(k \log d / \rho^2)$ samples suffice for this problem, with an inefficient algorithm (see Appendix \ref{app:info}).
Our first result is that efficient robust sparse PCA detection is possible, at effectively the best computationally efficient rate:

\begin{theorem}[Robust sparse PCA detection]
\label{thm:rspca1}
Fix $\rho, \delta, \eps > 0$.
Let $\eta = O(\eps \sqrt{\log 1 / \eps})$.
 Then, if $\eta = O(\rho)$, and we are given a we are given a $\eps$-corrupted set of samples from either $\normal(0, I)$ or $\normal (0, I + \rho vv^T)$ for some $k$-sparse unit vector $v$ of size
\[
n = \Omega \left( \frac{\min(d, k^2) + \log \binom{d^2}{k^2} + \log 1 / \delta} {\rho^2}\right)
\]
then there is a polynomial time algorithm which succeeds with probability $1 - \delta$ for detection.
\end{theorem}

It was shown in \cite{BR13} that even without noise, at least $n = \Omega (k^2 \log d / \eps^2)$ samples are required for any polynomial time algorithm for detection, under reasonable complexity theoretic assumptions.
Up to log factors, we recover this rate, even in the presence of noise.

We next consider the \emph{recovery} problem.
Here, we are given an $\eps$-corrupted set of samples from $\normal (0, I + \rho vv^T)$, and our goal is to output a $u$ minimizing $L(u, v)$, where $L(u, v) = \frac{1}{\sqrt{2}} \| uu^T -vv^T \| $.
For the recovery problem, we recover the following efficient rate:

\begin{theorem}[Robust sparse PCA recovery]
\label{thm:rspca2}
Fix $\eps, \rho > 0$.
Let $\eta$ be as in Theorem \ref{thm:rspca1}.
There is an efficient algorithm, which given a set of $\eps$-corrupted samples of size $n$ from $\normal (0, I + \rho vv^T)$, where 
\[
n = \Omega \left( \frac{\min(d, k^2) + \log \binom{d^2}{k^2} + \log 1 / \delta} {\eta^2}\right) \; ,
\]
outputs a $u$ so that
\[
L(u, v) = O \left( \frac{(1 + \rho) \eta}{\rho} \right) \; .
\]
\end{theorem}

In particular, observe that when $\eta = O(\rho)$, so when $\eps = \Otilde (\rho)$, this implies that we recover $v$ to some small constant error.
Therefore, given the same number of samples as in Theorem \ref{thm:rspca1}, this algorithm begins to provide non-trivial recovery guarantees.
Thus, this algorithm has the right ``phase transition'' for when it begins to work, as this number of samples is likely necessary for any computationally efficient algorithm.
Moreover, our rate itself is likely optimal (up to log factors), when $\rho = O(1)$.
In the non-robust setting, \cite{WBS16} showed a rate of (roughly) $O(\eps/ \rho)$ with the same number of samples, and that any computationally efficient algorithm cannot beat this rate.
We leave it as an interesting open problem to show if this rate is achievable or not in the presence of error when $\rho = \omega (1)$.

\section{Preliminaries}
In this section we provide technical preliminaries that we will require throughout the paper.
\subsection{Naive pruning}
We will require the following (straightforward) preprocessing subroutine from \cite{DKKLMS16} to remove all points which are more than $\Omegatilde (d)$ away from the true mean.

\begin{fact}[c.f. Fact 4.18 in \cite{DKKLMS16}]
\label{fact:naive-prune}
Let $X_1, \ldots, X_n$ be an $\eps$-corrupted set of samples from $\normal (\mu, I)$, and let $\delta > 0$.
There is an algorithm $\textsc{NaivePrune} (X_1, \ldots, X_n, \delta)$ which runs in $O(\eps d^2 n^2)$ time so that with probability $1 - \delta$, we have that (1) \textsc{NaivePrune} removes no uncorrupted points, and (2) if $X_i$ is not removed by \textsc{NaivePrune}, then $\| X_i - \mu \|_2 \leq O(\sqrt{d \log (n / \delta)})$.
If these two conditions happen, we say that \textsc{NaivePrune} has \emph{succeeded}.
\end{fact}

\subsection{Concentration inequalities}
In this section we give a couple of concentration inequalities that we will require in the remainder of the paper.
These ``per-vector'' and ``per-matrix'' concentration guarantees are well-known and follow from (scalar) Chernoff bounds, see e.g. \cite{DKKLMS16}.
\begin{fact}[Per-vector Gaussian concentration]
\label{fact:per-vector}
Fix $\eps, \delta > 0$.
Let $v \in \R^d$ be a unit vector, and let $X_1, \ldots, X_n \sim \normal (0, I)$, where
\[
n = \Omega \left( \frac{\log 1 / \delta}{\eps^2} \right) \; .
\]
Then, with probability $1 - \delta$, we have
\[
\left| \left\langle \frac{1}{n} \sum_{i = 1}^n X_i, v \right\rangle \right| \leq \eps \; .
\]
\end{fact}
\begin{fact}[Per-matrix Gaussian concentration]
\label{fact:per-matrix}
Fix $\eps, \delta > 0$, and suppose $\eps \leq 1$.
Let $M \in \R^{d \times d}$ be a symmetric matrix, and let $X_1, \ldots, X_n \sim \normal (0, I)$, where
\[
n = \Omega \left( \frac{\log 1 / \delta}{\eps^2} \right) \; .
\]
Then, with probability $1 - \delta$, we have:
\[
\left| \left\langle \frac{1}{n} \sum_{i = 1}^n X_i X_i^T - I, M \right\rangle \right| \leq \eps \; .
\]
\end{fact}

\subsection{The set $S_{n, \eps}$}
For any $n, \eps$, define the set
\begin{equation}
\label{eq:Sne}
S_{n, \eps} = \left\{ w \in \R^n: \sum_{i = 1}^n w_i = 1 \mbox{, and } 0 \leq w_i \leq \frac{1}{(1 - \eps) n}, \forall i \right\} \; .
\end{equation}
We make the following observation.
For any subset $I \subseteq [n]$, if we let $w^I$ be the vector whose $i$th coordinate is $1/|I|$ if $i \in I$ and $0$ otherwise, we have
\[
S_{n, \eps} = \mathrm{conv} \left\{ w_I : |I| = (1 - \eps) n \right\} \; .
\]
The set $S_{n, \eps}$ will play a key role in our algorithms.
We will think of elements in $S_{n, \eps}$ as weights we place upon our sample points, where higher weight indicates a higher confidence that the sample is uncorrupted, and a lower weight will indicate a higher confidence that the sample is corrupted.

\section{Concentration for sparse estimation problems via dual norms}
\label{sec:dual}
In this section we give a clean way of proving concentration bounds for various objects which arise in sparse PCA and sparse mean estimation problems.
We do so by observing they are instances of a very general ``meta-algorithm'' we call dual norm maximization.
This will prove crucial to proving the correctness of our algorithms for robust sparse recovery.
While this may sound similar to the ``dual certificate'' techniques often used in the sparse estimation literature, these techniques are actually quite different.
\begin{definition}[Dual norm maximization]
Let $\mathcal{H}$ be a Hilbert space with inner product $\langle \cdot, \cdot \rangle$.
Fix any set $S \subseteq \mathcal{H}$.
Then the dual norm induced by $S$, denoted $\| \cdot \|_S^*$, is defined by $\| x \|_S^* = \sup_{y \in S} | \langle x, y \rangle | $.
The dual norm maximizer of $x$, denoted $d_S (x)$, is the vector $d_S (x) = \argmax_{v \in S} | \langle v, x \rangle |$.
\end{definition}

In particular, we will use the following two sets.
Equip the space of symmetric $d \times d$ matrices with the trace inner product, i.e., $\langle A, B \rangle = \tr (A B)$, so that it is a Hilbert space, and let
\begin{align}
\U_k &= \{u \in \R^d: \| u \|_2 = 1, \| u \|_0 = k \} \label{eq:Uk} \\
\X_k &= \{X \in \R^{d \times d}: \tr(X) = 1, \| X \|_1 \leq k, X \succeq 0 \} \;  \label{eq:Xk} \; .
\end{align}

We show in Appendix \ref{app:dualnorm} that existing well-known algorithms for sparse mean recovery and sparse PCA without noise can be naturally written in this fashion.

Another detail we will largely ignore in this paper is the fact that efficient algorithms for these problems can only approximately solve the dual norm maximization problem.
However, we explain in Appendix \ref{sec:numerical} why this does not affect us in any meaningful way.
Thus, for the rest of the paper we will assume we have access to the exact maximizer, and the exact value of the norm.

\subsection{Concentration for dual norm maximization}
\label{sec:conc}

We now show how the above concentration inequalities allow us to derive very strong concentration results for the dual norm maximization problem for $\U_k$ and $\X_k$.
Conceptually, we view these concentration results as being the major distinction between sparse estimation and non-sparse estimation tasks.
Indeed, these results are crucial for adapting the convex programming framework for robust estimation to sparse estimation tasks.
Additionally, they allow us to give an easy proof that the $L_1$ relaxation works for sparse PCA.

\begin{corollary}
\label{cor:Uk-conc}
Fix $\eps, \delta > 0$.
Let $X_1, \ldots, X_n \sim \normal (0, I)$, where
\[
n = \Omega \left( \frac{k + \log \binom{d}{k} + \log 1 / \delta}{\eps^2} \right) \; .
\]
Then $\| \frac{1}{n} \sum_{i = 1}^n X_i \|^*_{\U_k} \leq \eps$.
\end{corollary}
\begin{proof}
Fix a set of $k$ coordinates, and let $S$ be the set of unit vectors supported on these $k$ coordinates.
By Fact \ref{fact:per-vector} and a net argument, one can show that for all $\delta$, given $n = \Omega \left( \frac{k + \log 1 / \delta}{\eps^2} \right)$, we have that 
\[
\left| \left\langle v,  \frac{1}{n} \sum_{i = 1}^n X_i \right\rangle \right| \leq \eps \; ,
\]
with probability $1 - \delta$.
The result then follows by setting $\delta' = \binom{d}{k}^{-1} \delta$ and union bounding over all sets of $k$ coordinates.
\end{proof}
\noindent The second concentration bound, which bounds deviation in $\X_k$ norm, uses ideas which are similar at a high level, but requires a bit more technical work.

\begin{theorem}
\label{thm:Xk-conc}
Fix $\eps, \delta > 0$.
Let $X_1, \ldots X_n \sim \normal (0, I)$, where
\[
n = \Omega \left( \frac{\min(d, k^2) + \log \binom{d^2}{k^2} + \log 1 / \delta}{\eps^2} \right) \; .
\]
Then
\[
\left\| \frac{1}{n} \sum_{i = 1}^n X_i X_i^T - I \right\|^*_{\X_k} \leq \eps \; .
\]
\end{theorem}
\noindent Let us first introduce the following definition.
\begin{definition}
A \emph{symmetric sparsity pattern} is a set $S$ of indices $(i, j) \in [d] \times [d]$ so that if $(i, j) \in S$ then $(j, i) \in S$.
We say that a symmetric matrix $M \in \R^{d \times d}$ respects a symmetric sparsity pattern $S$ if $\supp (M) = S$.
\end{definition}
\noindent With this definition, we now show:
\begin{lemma}
\label{lem:k2-sparse}
Let $n = O \left( \frac{\min(d, k^2) + \log \binom{d^2}{k^2} + \log 1 / \delta}{\eps^2} \right)$.
Then, with probability $1 - \delta$, the following holds:
\begin{equation}
\label{eq:sparse-conc}
| \tr ((\Sigmahat - I) X)| \leq O(\eps),~\mbox{for all symmetric $X$ with $\| X \|_0 = k^2$ and $\| X \|_F \leq 1$.} 
\end{equation}
\end{lemma}
\begin{proof}
Fix any symmetric sparsity pattern $S$ so that $|S| \leq k^2$.
By classical arguments one can show that there is a $(1/3)$-net over all symmetric matrices $X$ with $\| X \|_F = 1$ respecting $S$ of size at most $9^{O(\min(d,k^2))}$.
By Fact \ref{fact:per-matrix} and a basic net argument, we know that for any $\delta'$, we know that except with probability $1 - \delta'$, if we take $n = O \left( \frac{\min(d, k^2) + \log 1 / \delta'}{\eps^2} \right)$ samples, then for all symmetric $X$ respecting $S$ so that $\| X \|_F \leq 1$, we have $| \tr ((\Sigmahat - I) X)| \leq \eps$.
The claim then follows by further union bounding over all $O \left( \binom{d^2}{k^2} \right)$ symmetric sparsity patterns $S$ with $|S| \leq k^2$.
\end{proof}
\noindent We will also require the following structural lemma.
\begin{lemma}
\label{lem:l0-to-l1}
Any PSD matrix $X$ so that $\tr(X) = 1$ and $\| X \|_1 \leq k$ can be written as
\[
X = \sum_{i = 1}^{O(n^2 / k^2)} Y_i \; ,
\]
where each $Y_i$ is symmetric, have $\sum_{i = 1}^{O(n^2 / k^2)} \| Y_i \|_F \leq 4$, and each $Y_i$ is $k^2$-sparse.
\end{lemma}
\begin{proof}
Observe that since $X$ is PSD, then $\| X \|_F \leq \tr(X) = 1$. 
For simplicity of exposition, let us ignore that the $Y_i$ must be symmetric for this proof.
We will briefly mention how to in addition ensure that the $Y_i$ are symmetric at the end of the proof.
Sort the entries of $X$ in order of decreasing $|X_{ij}|$.
Let $Y_i$ be the matrix whose nonzeroes are the $ik^2 +1$ through $(i + 1) k^2$ largest entries of $X$, in the same positions as they appear in $X$.
Then we clearly have that $\sum Y_i = X_i$, and each $Y_i$ is exactly $k^2$-sparse.\footnote{Technically the last $Y_i$ may not be $k^2$ sparse but this is easily dealt with, and we will ignore this case here}
Thus it suffices to show that $\sum \| Y_i \|_F \leq 4$.
We have $\| Y_1 \|_F \leq \| X \|_F \leq 1$.
Additionally, we have $\| Y_{i + 1} \|_F \leq \frac{1^T  |Y_i | 1}{k}$, which follows simply because every nonzero entry of $Y_{i + 1}$ is at most the smallest entry of $Y_i$, and each has exactly $k^2$ nonzeros (except potentially the last one, but it is not hard to see this cannot affect anything).
Thus, in aggregate we have
\[
\sum_{i = 1}^{O(n^2 / k^2)} \| Y_i \|_F \leq 1 + \sum_{i = 2}^{O(n^2 / k^2)} \| Y_i \|_F \leq 1 + \sum_{i = 1}^{O(n^2 / k^2)} \frac{1^T |Y_i| 1}{k} = 1 + \frac{1^T |X| 1}{k} \leq 2 \; ,
\]
which is stronger than claimed.

However, as written it is not clear that the $Y_i$'s must be symmetric, and indeed they do not have to be.
The only real condition we needed was that the $Y_i$'s (1) had disjoint support, (2) summed to $X$, (3) are each $\Theta (k^2)$ sparse (except potentially the last one), and (4) the largest entry of $Y_{i + 1}$ is bounded by the smallest entry of $Y_i$.
It should be clear that this can be done while respecting symmetry by doubling the number of $Y_i$, which also at most doubles the bound in the sum of the Frobenius norms. 
We omit the details for simplicity.
\end{proof}
\begin{proof}[Proof of Theorem \ref{thm:Xk-conc}]
Let us condition on the event that (\ref{eq:sparse-conc}) holds.
We claim then that for all $X \in \X$, we must have $|\tr ((\Sigmahat - I) X)| \leq O(\eps)$, as claimed.
Indeed, by Lemma \ref{lem:l0-to-l1}, for all $X \in \X$, we have that
\[
X = \sum_{i = 1}^{O(d^2 / k^2)} Y_i \; ,
\]
where each $Y_i$ is symmetric, have $\sum_{i = 1}^{O(d^2 / k^2)} \| Y_i \|_F \leq 4$, and each $Y_i$ is $k^2$-sparse.
Thus,
\begin{align*}
|\tr ((\Sigmahat - I) X)| &\leq \sum_{i = 1}^{O(d^2 / k^2)} \left| \tr ((\Sigmahat - I) Y_i) \right| \\
&= \sum_{i = 1}^{O(d^2 / k^2)} \| Y_i \|_F \left| \tr \left( (\Sigmahat - I) \frac{Y_i}{\| Y_i \|_F} \right) \right| \\
&\stackrel{(a)}{\leq} \sum_{i = 1}^{O(d^2 / k^2)} \| Y_i \|_F \cdot O(\eps) \\
&\stackrel{(b)}{\leq} O(\eps) \; ,
\end{align*}
where (a) follows since each $Y_i / \| Y_i \|_F$ satisfies the conditions in (\ref{eq:sparse-conc}), and (b) follows from the bound on the sum of the Frobenius norms of the $Y_i$.
\end{proof}

\subsection{Concentration for $S_{n, \eps}$}

We will require the following concentration inequalities for weighted sums of Gaussians, where the weights come from $S_{n, \eps}$, as these objects will naturally arise in our algorithms.
These bounds follow by applying the above bounds, then carefully union bounding over all choices of possible subsets of $\binom{n}{\eps n}$ subsets.
We need to be careful here since the number of things we are union bounding over increases as $n$ increases.
We include the proofs in Appendix \ref{app:proofs-dual}.

\begin{theorem}
\label{thm:Uk-Sne-conc}
Fix $\eps \leq 1/2 $ and $\delta \leq 1$, and fix $k \leq d$.
There is a $\eta_1 = O(\eps \sqrt{\log 1 / \eps})$ so that for any $\eta > \eta_1$, if $X_1, \ldots, X_n \sim \normal (0, I)$ and $n = \Omega \left( \frac{\min(d, k^2) + \log \binom{d^2}{k^2} + \log 1 / \delta} {\eta^2}\right)$, then
\[
\Pr \left[ \exists w \in S_{n, \eps} : \left\| \frac{1}{n} \sum_{i = 1}^n w_i X_i \right\|^*_{\U_k} \geq \eta \right] \leq \delta \; .
\]
\end{theorem}

\begin{theorem}
\label{thm:Xk-Sne-conc}
Fix $\eps \leq 1/2 $ and $\delta \leq 1$, and fix $k \leq d$.
There is a $\eta = O(\eps \sqrt{\log 1 / \eps})$ so that if $X_1, \ldots, X_n \sim \normal (0, I)$ and $n = \Omega \left( \frac{\min(d, k^2) + \log \binom{d^2}{k^2} + \log 1 / \delta} {\eta^2}\right)$, then we have
\[
\Pr \left[ \exists w \in S_{n, \eps} : \left\| \frac{1}{n} \sum_{i = 1}^n w_i X_i X_i^T - I \right\|^*_{\X_k} \geq \eta \right] \leq \delta \; .
\]
\end{theorem}

\section{A robust algorithm for robust sparse mean estimation}
\label{sec:sparse-mean}
This section is dedicated to the description of an algorithm \textsc{RecoverRobustSMean} for robustly learning Gaussian sequence models, and the proof of the following theorem:
\begin{theorem}
\label{thm:robust-GSM}
Fix $\eps, \tau >0$. Let $\eta = O(\eps \sqrt{\log 1 / \eps})$. Given an $\eps$-corrupted set of samples of size $n$ from $\normal (\mu, I)$, where $\mu$ is $k$-sparse
\[
n = \Omega \left( \frac{\min(k^2, d) + \log \binom{d^2}{k^2} + \log 1 / \tau}{\eta^2} \right) \; ,
\]
then \textsc{RecoverRobustSMean} outputs a $\muhat$ so that with probability $1 - \tau,$ we have $\| \muhat - \mu \|_2 \leq O(\eta)$.
\end{theorem}

Our algorithm builds upon the convex programming framework developed in \cite{DKKLMS16}.
Roughly speaking, the algorithm proceeds as follows.
First, it does a simple naive pruning step to remove all points which are more than roughly $\Omega (\sqrt{d})$ away from the mean.
Then, for an appropriate choice of $\delta$, it will attempt to (approximately) find a point within the following convex set:
\begin{equation}RobustSMean
\label{eq:feasible}
C_\tau = \left\{ w \in S_{n, \eps}: \left\| \sum_{i = 1}^n w_i (X_i - \mu) (X_i - \mu)^T - I \right\|^*_{\X_k} \leq \tau \right\} \; .
\end{equation}
The main difficulty with finding a point in $C_\tau$ is that $\mu$ is unknown.
A key insight of \cite{DKKLMS16} is that it suffices to create an (approximate) separation oracle for the feasible set, as then we may use classical convex optimization algorithms (i.e. ellipsoid or cutting plane methods) to find a feasible point.
In their setting (for a different $C_\tau$), it turns out that a simple spectral algorithm suffices to give such a separation oracle.

Our main contribution is the design of separation oracle for $C_\tau$, which requires more sophisticated techniques.
In particular, we will ideas developed in analogy to hard thresholding and SDPs similar to those developed for sparse PCA to design such an oracle.

\subsection{Additional preliminaries}
Throughout this section, we let $X_1, \ldots, X_n$ denote an $\eps$-corrupted set of samples from $\normal (\mu, I)$, where $\mu$ is $k$-sparse.
We let $\Sgood$ denote the set of uncorrupted samples, and we let $\Sbad$ denote the set of corrupted samples.
For any set of weights $w \in S_{n, \eps}$, we let $\wg = \sum_{i \in \Sgood} w_i$  and $\wb = \sum_{i \in \Sbad} w_i$.

Throughout this section, we will condition on the following three deterministic events occurring:
\begin{align}
\textsc{NaivePrune}(X_1, \ldots, X_n, \delta)~&\mbox{succeeds,} \label{eq:cond1} \\
\left\| \sum_{i \in \Sgood} w_i (X_i - \mu) \right\|^*_{\U_k} &\leq \eta \; , ~\forall w \in S_{n, 2 \eps} \; \mbox{, and} \label{eq:cond2} \\
\left\| \sum_{i \in \Sgood} w_i (X_i - \mu) (X_i - \mu)^T - \wg I \right\|^*_{\X_k} &\leq \eta \; , ~\forall w \in S_{n, 2\eps} \; , \label{eq:cond3}
\end{align}
where $\eta= O(\eps \sqrt{\log 1 / \eps})$.
When $n = \Omega \left( \frac{\min(k^2, d) + \log \binom{k^2}{d^2} + \log 1 / \delta}{\eta^2} \right)$ these events simultaneously happen with probability at least $1 - O(\delta)$ by Fact \ref{fact:naive-prune}, Theorem \ref{thm:Uk-Sne-conc}, Theorem \ref{thm:Xk-Sne-conc} and a union bound, and the observation that if $w \in S_{n, \eps}$, then $w / w^g$ restricted to the indices in $\Sgood$ is in $S_{(1 - \eps) n, 2 \eps}$.

\subsection{The separation oracle}
Our main result in this section is the description of a polynomial time algorithm \textsc{RobustSMeanOracle} and the proof of the following theorem of its correctness:
\begin{theorem}
\label{thm:robustGSM}
Fix $\eps > 0$ sufficiently small.
Suppose that (\ref{eq:cond2}) and (\ref{eq:cond3}) hold.
Let $w^*$ denote the set of weights which are uniform over the uncorrupted points.
Then, there is a constant $1 \leq c \leq 21$ so that \textsc{RobustSMeanOracle} satisfies:
\begin{enumerate}
\item (Completeness) If $w = w^*$, \textsc{RobustSMeanOracle} outputs ``YES''.
\item (Soundness) If $w \not\in C_{c \eta}$ the algorithm outputs a hyperplane $\ell: \R^n \to \R$ so that $\ell (w) \geq 0$ but $\ell (w^*) < 0$.
Moreover, if the algorithm ever outputs a hyperplane, we have $\ell(w^*) < 0$.
\end{enumerate}
\end{theorem}

Plugging these guarantees into an ellipsoid (or cutting-plane) method, we obtain the following:
\begin{corollary}
\label{cor:approxRecover}
Fix $\eps > 0$ sufficiently small.
Suppose that (\ref{eq:cond2}) and (\ref{eq:cond3}) hold.
There is an algorithm \textsc{ApproxRecoverRobustSMean} which queries \textsc{RobustSMeanOracle} at most $\poly(d, 1/\eps, \log 1/\delta)$ times, and so runs in time $\poly (d, 1/ \eps, 1/\delta)$ which outputs a $w'$  so that $\| w - w'\|_\infty \leq \eps / (n \sqrt{d \log n / \delta})$, for some $w \in C_{c \tau}$.
\end{corollary}

Our separation oracle, formally described in Algorithm \ref{alg:GSM}, proceeds as follows.
Given $w \in S_{n, \eps}$, it forms $\muhat = \| \muhat' \|^*_{\U_k} \cdot d_{\U_k} (\muhat')$, where $\muhat = \sum w_i X_i$.
It then forms the matrix $\Sigmahat = \sum w_i (X_i - \muhat) (X_i - \muhat)^T$, and computes $A = d_{\X_k} (\Sigmahat)$.
The algorithm then checks if $\left| \langle A, \Sigmahat \rangle \right| > C$ for appropriately chosen threshold $C$.
If it does not, the algorithm outputs ``YES''.
Otherwise, the algorithm outputs a separating hyperplane given by this matrix $A$.

\begin{algorithm}[htb]
\begin{algorithmic}[1]
\Function{RobustSMeanOracle}{$X_1, \ldots, X_n$, $w$}
\State Let $\muhat =  \sum w_i X_i$
\State Let $\Sigmahat = \sum w_i (X_i - \muhat) (X_i - \muhat)^T$
\State Let $A = d_{\X_k} (\Sigmahat)$
\State {\bf if} $|\langle A, \Sigmahat - I \rangle | \geq 20 \eta$ {\bf then}
	\State ~~~~~Let $\sigma = \sgn{\langle A, \Sigmahat - I \rangle}$
	\State ~~~~~{\bf return} the hyperplane $\ell$ given by
	\[
	\ell (w) = \sigma \left( \sum_{i = 1}^n w_i \left\langle A, (X_i - \muhat)  (X_i - \muhat)^T \right\rangle - 1 \right) - |\langle A, \Sigmahat - I \rangle | \; .
	\]
\State {\bf else}
	\State ~~~~{\bf return} ``YES''
\State {\bf end}
\EndFunction
\end{algorithmic}
\caption{Separation oracle for robust sparse mean estimation.}
\label{alg:GSM}
\end{algorithm}

We will require the following two lemmata:
\begin{lemma}
\label{lem:sos}
Let $\omega_1, \ldots, \omega_m$ be a set of non-negative weights that sum to 1.
Let $a_1, \ldots, a_m$ be any sequence of scalars.
Then
\[
\sum_{i = 1}^m \omega_i a_i^2 \geq \left( \sum_{i = 1}^m \omega_i a_i \right)^2 \; .
\]
\end{lemma}

\begin{proof}
Let $Z$ be a random variable which is $a_i$ with probability $\omega_i$.
Then $\E [Z] = \sum \omega_i a_i$ and $\E [Z^2] =  \sum \omega_i a_i^2$.
Then the inequality follows from the fact that $\E [Z^2] - \E [Z^2] = \Var [Z] \geq 0$.
\end{proof}

\begin{lemma}
\label{lem:Xk-to-Uk}
Let $u \in \R^d$.
Then $(\| u \|^*_{\U_k})^2 \leq \| u u^T \|^*_{\X_k} \leq 4 (\| u \|_{\U_k}^*)^2$.
\end{lemma}
\begin{proof}
Let $v = d_{\U_k} (u)$.
Then since $v v^T \in \X_k$, we have that $ (\| u u^T \|_{\X_k}^*) \geq \langle v v^T, u u^T \rangle = \langle u, v \rangle^2 = (\| u \|_{\U_k}^*)^2$.
This proves the first inequality.

To prove the other inequality, we first prove the intermediate claim that $\sup_{M \in \Y_{k^2}} u^T M u \leq (\| u \|_{\U_k}^*)^2$, where $\Y_{k^2}$ is the set of symmetric matrices $M$ with at most $k^2$-non-zeroes satisfying $\| M \|_F = 1$.
Indeed, fix any $M \in \Y_k$.
Let $S \subseteq[n]$ be the set of non-zeroes of $d_{\U_k} (u)$.
This is exactly the set of the $k$ largest elements in $u$, sorted by absolute value.
Let $P$ be the symmetric sparsity pattern respected by $M$.
Fix an arbitrary bijection $\phi: P \setminus (S \times S) \to (S \times S) \setminus P$,
and let $M'$ be the following matrix:
\[
M'_{i,j} = \left\{ \begin{array}{ll}
         M_{ij} & \mbox{if $(i, j) \in P \bigcap (S \times S)$} \; ,\\
        \sgn{u_i u_j} M_{\phi^{-1} (i, j)} & \mbox{if $(i, j) \in  (S \times S) \setminus P$} \; ,\\
        0 & \mbox{otherwise}.\end{array} \right.
\]
Then we claim that $u^T M u \leq u^T M' u$.
Indeed, we have 
\begin{align*}
u^T M' u - u^T M u &= \sum_{(i, j) \in P \setminus (S \times S)} |M_{ij} (uu^T)_{\phi(i, j)}| - M_{ij} (uu^T)_{i, j} \\
&\geq \sum_{(i, j) \in P \setminus (S \times S)} |M_{i, j}| \left( |(uu^T)_{\phi(i, j)}| - |(uu^T)_{i, j}| \right) \geq 0 \; ,
\end{align*}
from the definition of $S$.
Moreover, for any $M$ respecting $S \times S$ with $\| M \|_F = 1$, it is not hard to see that $u^T M u \leq (\| u \|_{\U_k}^*)^2$.
This is because now the problem is equivalent to restricting our attention to the coordinates in $S$, and asking for the symmetric matrix $M \in \R^{S \times S}$ with $\| M \|_F = 1$ maximizing $u_S^T M u_S$, where $u_S$ is $u$ restricted to the coordinates in $S$.
This is clearly maximized by $M = \frac{1}{\| u_S\|_2^2}u_S u_S^T$, which yields the desired expression, since $\| u_S \|_2 = \| u \|_{\U_k}$.

We can now prove the original lemma.
By Lemma \ref{lem:l0-to-l1} we may write $A = \sum_{i = 1}^{O(n^2 / k^2)} Y_i$ where each $Y_i$ is symmetric, $k^2$-sparse, and have $\sum_{i = 1}^{O(n^2 / k^2)} \| Y_i \|_F \leq 4$.
We therefore have
\begin{align*}
u^T A u &= \sum_{i = 1}^{O(n^2 / k^2)} u^T Y_i u \\
&= \sum_{i = 1}^{O(n^2 / k^2)} \| Y_i \|_F (\| u \|_{\U_k}^*)^2 \\
&\leq 4 (\| u \|_{\U_k}^*)^2 \; ,
\end{align*}
as claimed, where the second line follows from the arguments above.
\end{proof}

Throughout the rest of this section, let $Y_i = X_i - \mu$, so that so that $Y_i \sim \normal (0, I)$ if $i \in \Sgood$.
We first prove the following crucial proposition:

\begin{proposition}
\label{prop:main}
Let $w \in S_{n, \eps}$, and let $\tau \geq \eta$.
Assuming (\ref{eq:cond2}) and (\ref{eq:cond3}) hold, if $\left\| \sum_{i = 1}^n w_i Y_i \right\|_{\U_k}^* \geq 3 \tau$, then $\left\| \sum_{i = 1}^n w_i Y_i Y_i^T - I \right\|_{\X_k}^* \geq \frac{\tau^2}{\eps}$.
\end{proposition}
\begin{proof}
Observe that (\ref{eq:cond2}) and a triangle inequality together imply that $\left\| \sum_{i \in \Sbad} w_i Y_i \right\|_{\U_k}^* \geq 2 \tau$.
By definition, this implies there is a $k$-sparse unit vector $u$ so that $\left| \langle u, \sum_{i \in \Sbad} w_i Y_i \rangle \right| \geq 2 \tau$.
WLOG assume that $\langle u, \sum_{i \in \Sbad} w_i Y_i \rangle \geq \eta$ (if the sign is negative a symmetric argument suffices).
This is equivalent to the statement that 
\[
\sum_{i \in \Sbad} \frac{w_i}{w^b} \langle u, Y_i \rangle \geq \frac{2 \tau}{w^b} \; .
\]
Observe that the $w_i / w^b$ are a set of non-negative weights summing to $1$.
Hence, by Lemma \ref{lem:sos}, we have
\[
\sum_{i \in \Sbad} \frac{w_i}{w^b} \langle u, Y_i \rangle^2 \geq \left( \frac{2 \tau}{w^b} \right)^2 \; .
\]
Let $A = u u^T$.
Observe that $A \in \X_k$.
Then the above inequality is equivalent to the statement that
\[
\sum_{i \in \Sbad} w_i Y_i^T A Y_i \geq \frac{\tau^2}{w^b} \geq \frac{4 \tau^2}{\eps}\; .
\]
Moreover, by (\ref{eq:cond3}), we have
\[
\left| \sum_{i \in \Sgood} w_i Y_i^T A Y_i - I \right| \leq \eta \; ,
\]
and together these two inequalities imply that
\[
\sum_{i = 1}^n w_i Y_i A Y_i \geq \frac{4 \tau^2}{\eps} - \eta \geq \frac{\tau^2}{\eps} \; ,
\]
as claimed.
The final inequality follows from the definition of $\eta$, and since $4 > 2$.
\end{proof}

\begin{proof}[Proof of Theorem \ref{thm:robustGSM}]
Completeness follows from (\ref{eq:cond3}).
We will now show soundness.
Suppose $w \not\in C_{21 \eta}$.
We wish to show that we will output a separating hyperplane.
From the description of the algorithm, this is equivalent to showing that $\| \Sigmahat - I \|_{\X_k} \geq 20 \eta$.
Let $\muhat = \sum_{i = 1}^n w_i X_i$, and let $\Delta = \mu - \muhat$.
By elementary manipulations, we may write
\begin{align*}
\left\| \sum_{i = 1}^n w_i (X_i - \muhat) (X_i - \muhat)^T - I \right\|_{\X_k} &= \left\| \sum_{i = 1}^n w_i (Y_i + \Delta) (Y_i + \Delta)^T - I \right\|_{\X_k} \\
&\stackrel{(a)}{=} \left\| \sum_{i = 1}^n w_i Y_i Y_i^T + \Delta \Delta^T - I \right\|_{\X_k} \\
&\stackrel{(b)}{\geq} \left\| \sum_{i = 1}^n w_i Y_i Y_i^T - I \right\|_{\X_k} - \left\| \Delta \Delta^T \right\|_{\X_k} \\
&\stackrel{(c)}{\geq}  \left\| \sum_{i = 1}^n w_i Y_i Y_i^T - I \right\|_{\X_k} - 4 \left\| \Delta \right\|_{\U_k}^2 \; ,
\end{align*}
where (a) follows since $\sum_{i = 1}^n w_i Y_i = \Delta$ by definition, (b) follows from a triangle inequality, and (c) follows from Lemma \ref{lem:Xk-to-Uk}.
If $\| \Delta \|_{\U_k} \leq \sqrt{\eta / 2}$, then the RHS is at least $21 \eta$ since the second term is at most $\eta$, and the first term is at least $21 \eta$ since we assume that $w \not\in C_{21 \eta}$.
Conversely, if $\| \Delta \|_{\U_k} \geq \sqrt{\eta / 2}$, then by Proposition \ref{prop:main}, we have $\| \sum_{i = 1}^n w_i Y_i Y_i - I \|_{\X_k} \geq \| \Delta \|_{\X_k}^2 / (6 \eps) > 48 \| \Delta \|_{\X_k}^2$ as long as $\eps \leq 1 / 288$.
This implies that the RHS is at least $40 \| \Delta \|_{\X_k^2} \geq 20 \eta$, as claimed.

Hence, this implies that if $w \not\in C_{4 \eta}$, then we output a hyperplane $\ell$.
It is clear by construction that $\ell (w) \geq 0$; thus, it suffices to show that if we output a hyperplane, that $\ell (w^*) < 0$.
Letting $\mutilde = \frac{1}{(1 - \eps) n} \sum_{i \in \Sgood} w_i Y_i$, we have
Observe that we have
\begin{align*}
\sum_{i = 1}^n w_i^* (X_ i - \muhat) (X_i - \muhat)^T - I &= \frac{1}{(1 - \eps) n}\sum_{i \in \Sgood} (Y_i + \Delta) (Y_i + \Delta)^T - I \\
&= \frac{1}{(1 - \eps) n} \left( \sum_{i \in \Sgood} Y_i Y_i^T - I \right) + \Delta \mutilde^T + \mutilde \Delta^T + \Delta \Delta^T \\
&= \frac{1}{(1 - \eps) n } \left( \sum_{i \in \Sgood} Y_i Y_i^T - I \right) + (\Delta + \mutilde) (\Delta + \mutilde)^T - \mutilde \mutilde^T \; .
\end{align*}
Hence by the triangle inequality and Lemma \ref{lem:Xk-to-Uk}, we have
\begin{align*}
\left\| \sum_{i = 1}^n w_i^* (X_ i - \muhat) (X_i - \muhat)^T - I \right\|_{\X_k} &\leq \left\| \frac{1}{1 (1 - \eps)n } \sum_{i \in \Sgood} Y_i Y_i^T - I \right\|^*_{\X_k} \\
&~~~~+ 4  \left( \left\| \Delta + \mutilde \right\|_{\U_k}^* \right)^2 + 4 \left( \left\| \mutilde \right\|_{\U_k}^* \right)^2 \\
&\leq \left\| \frac{1}{1 (1 - \eps)n } \sum_{i \in \Sgood} Y_i Y_i^T - I \right\|_{\X_k} + 8  \left( \left\| \Delta \right\|^*_{\U_k} \right)^2 \\
&~~~~+ 8 \left( \left\| \mutilde \right\|_{\U_k}^* \right)^2 + 4 \left( \left\| \mutilde \right\|_{\U_k}^* \right)^2 \\
&\leq 13 \eta + 8 \left( \left\| \Delta \right\|_{\U_k}^* \right)^2 \; , \numberthis \label{eq:wstarbound}
\end{align*}
by (\ref{eq:cond2}) and (\ref{eq:cond3}).

Observe that to show that $\ell(w^*) < 0$ it suffices to show that
\begin{equation}
\label{eq:soundness}
\left\| \sum_{i = 1}^n w^*_i (X_i - \muhat) (X_i - \muhat) - I \right\|_{\X_k}^* < \left\| \Sigmahat - I \right\|_{\X_k}^* \; .
\end{equation}
If $\| \Delta \|_{\U_k}^* \leq \sqrt{\eta / 2}$, then this follows since the quantity on the RHS is at least $20 \eta$ by assumption, and the quantity on the LHS is at most $17 \eta$ by (\ref{eq:wstarbound}).
If $\| \Delta \|_{\U_k}^* \geq \sqrt{\eta / 2}$, then by Proposition \ref{prop:main}, the RHS of (\ref{eq:soundness}) is at least $\left( \| \Delta \|_{\U_k}^* \right)^2 / (3 \eps)$, which dominates the LHS as long as $\| \Delta \|_{\U_k}^* \geq \eta$ and $\eps \leq 1 / 288$, which completes the proof.
\end{proof}

\subsection{Putting it all together}
We now have the ingredients to prove our main theorem.
Given what we have, our full algorithm \textsc{RecoverRobustSMean} is straightforward: first run \textsc{NaivePrune}, then run \textsc{ApproxRecoverRobustSMean} on the pruned points to output some set of weights $w$.
We then output $\| \muhat \|_{\U_k} d_{\U_k} (\muhat)$.
The algorithm is formally defined in Algorithm \ref{alg:robust-GSM}.

\begin{algorithm}[htb]
\begin{algorithmic}[1]
\Function{RecoverRobustSMean}{$X_1, \ldots, X_n, \eps, \delta$}
\State Let $S$ be the set output by $\textsc{NaivePrune} (X_1, \ldots, X_n, \delta)$. WLOG assume $S = [n]$.
\State Let $w' = \textsc{ApproxRecoverRobustSMean}(X_1, \ldots, X_n, \eps, \delta)$.
\State Let $\muhat = \sum_{i = 1}^n w'_i X_i$.
\State {\bf return} $\| \muhat \|^*_{\U_k} d_{\U_k} (\muhat)$
\EndFunction
\end{algorithmic}
\caption{An efficient algorithm for robust sparse mean estimation}
\label{alg:robust-GSM}
\end{algorithm}

\begin{proof}[Proof of Theorem \ref{thm:robust-GSM}]
Let us condition on the event that (\ref{eq:cond1}), (\ref{eq:cond2}), and (\ref{eq:cond3}) all hold simultaneously.
As previously mentioned, when $n = \Omega \left( \frac{\min(k^2, d) + \log \binom{k^2}{d^2} + \log 1 / \delta}{\eta^2} \right)$ these events simultaneously happen with probability at least $1 - O(\delta)$.
For simplicity of exposition, let us assume that $\textsc{NaivePrune}$ does not remove any points.
This is okay since if it succeeds, it never removes any good points, so if it removes any points, it can only help us.
Moreover, since it succeeds, we know that $\| X_i - \mu \|_2 \leq O(\sqrt{d \log (n / \delta)})$ for all $i \in [n]$.
By Corollary \ref{cor:approxRecover}, we know that there is some $w \in C_{21 \eta}$ so that $\| w - w' \|_\infty \leq \eps / (n \sqrt{d \log n / \delta})$.
We have
\begin{align*}
\left\| \muhat - \mu \right\|_{\U_k} = \left\| \sum_{i = 1}^n w_i' X_i - \muhat \right\|^*_{\U_k} &\leq \left\| \sum_{i = 1}^n w_i X_i - \muhat \right\|^*_{\U_k} + \sum_{i = 1}^n |w_i - w_i'| \left\| X_i - \mu \right\|_2 \\
&\leq O (\eta) + O(\eps) \; ,
\end{align*}
by Proposition \ref{prop:main}.
We now show that this implies that if we let $\mu' = \| \muhat \|^*_{\U_k} d_{\U_k} (\muhat)$, then $\| \mu' - \mu \|_2 \leq O(\eta)$.
Let $S$ be the support of $\mu'$, and let $T$ be the support of $\mu$.
Then we have 
\[
\| \mu' - \mu \|_2^2 = \sum_{i \in S \cap T} (\mu'_i - \mu_i)^2 + \sum_{i \in S \setminus T} (\mu_i')^2 + \sum_{i \in T \setminus S} \mu_i^2 \; .
\]
Observe that $ \sum_{i \in S \cap T} (\mu'_i - \mu_i)^2 + \sum_{i \in S \setminus T} (\mu_i')^2 \leq \left( \| \muhat - \mu \|_{\U_k}^* \right)^2$, since $\mu$ was originally nonzero on the entries in $S \setminus T$.
Moreover, for all $i \in T \setminus S$ and $j \in S \setminus T$, we have $(\mu'_i)^2 \leq (\mu'_j)^2$.
Thus we have
\[
\sum_{i \in T \setminus S} \mu_i^2 \leq 2 \left( \sum_{i \in T \setminus S} (\mu - \mu'_i)^2 + \sum_{i \in S \setminus T} (\mu'_j)^2 \right) \leq 2 \left( \| \muhat - \mu \|_{\U_k}^* \right)^2 \; .
\]
Therefore we have $\| \mu' - \mu \|_2^2 \leq 3 \left( \| \muhat - \mu \|_{\U_k}^* \right)^2$, which implies that $\| \mu' - \mu \|_2 \leq O(\eta)$, as claimed.
\end{proof}

\section{An algorithm for robust sparse PCA detection}
\label{sec:detection}
In this section, we give an efficient algorithm for detecting a spiked covariance matrix in the presence of adversarial noise.
Our algorithm is fairly straightforward: we ask for the set of weights $w \in S_{n, \eps}$ so that the empirical second moment with these weights has minimal deviation from the identity in the dual $\X_k$ norm.
We may write this as a convex program.
Then, we check the value of the optimal solution of this convex program.
If this value is small, then we say it is $\normal (0, I)$.
if this value is large, then we say it is $\normal (0, I + \rho vv^T)$.
We refer to the former as Case 1 and the latter as Case 2.
The formal description of this algorithm is given in Algorithm.

\begin{algorithm}[htb]
\begin{algorithmic}[1]
\Function{DetectRobustSPCA}{$X_1, \ldots, X_n, \eps, \delta, \rho$}
\State Let $\gamma$ be the value of the solution
\begin{equation}
\label{eq:SDP-SPCA1}
\min_{w \in S_{n, \eps}} \left\| \sum_{i = 1}^n w_i ( X_i X_i^T - I) \right\|^*_{\X_k}
\end{equation}
\State {\bf if} $\gamma < \rho / 2$ {\bf then} {\bf return} Case 1 {\bf else} {\bf return} Case 2
\EndFunction
\end{algorithmic}
\caption{Learning a spiked covariance model, robustly}
\label{alg:robustSPCA2}
\end{algorithm}

\subsection{Implementing \textsc{DetectRobustSPCA}}
\label{sec:implement}
We first show that the algorithm presented above can be efficiently implemented.
Indeed, one can show that by taking the dual of the SDP defining the $\| \cdot \|^*_{\X_k}$ norm, this problem can be re-written as an SDP with (up to constant factor blowups) the same number of constraints and variables, and therefore we may solve it using traditional SDP solver techniques.

Alternatively, one may observe that to optimize Algorithm \ref{alg:robustSPCA2} via ellipsoid or cutting plane methods, it suffices to, given $w \in S_{n, \eps}$, produce a separating hyperplane for the constraint (\ref{eq:SDP-SPCA1}).
This is precisely what dual norm maximization allows us to do efficiently.
It is straightforward to show that the volume of $S_{n, \eps} \times \X_k$ is at most exponential in the relevant parameters.
Therefore, by the classical theory of convex optimization, (see e.g. [CITE]), for any $\xi$, we may find a solution $w'$ and $\gamma'$ so that $\| w' - w^* \|_\infty \leq \xi$  and  $\gamma'$ so that $|\gamma - \gamma'| < \xi$ for some exact minimizer $w^*$, where $\gamma$ is the true value of the solution, in time $\poly (d, n, 1 / \eps, \log 1 / \xi)$,

As mentioned in Section \ref{sec:numerical}, neither approach will in general give exact solutions, however, both can achieve inverse polynomial accuracy in the parameters in polynomial time.
We will ignore these issues of numerical precision throughout the remainder of this section, and assume we work with exact $\gamma$.

Observe that in general it may be problematic that we don't have exact access to the minimizer $w^*$, since some of the $X_i$ may be unboundedly large (in particular, if it's corrupted) in norm.
However, we only use information about $\gamma$.
Since $\gamma$ lives within a bounded range, and our analysis is robust to small changes to $\gamma$, these numerical issues do not change anything in the analysis.

\subsection{Proof of Theorem \ref{thm:rspca1}}

We now show that Algorithm \ref{alg:robustSPCA2} provides the guarantees required for Theorem \ref{thm:rspca1}.
We first show that if we are in Case 1, then $\gamma$ is small:
\begin{lemma}
\label{lem:detect1}
Let $\rho, \delta > 0$.
Let $\eps, \eta$ be as in Theorem \ref{thm:rspca1}.
Let $X_1, \ldots, X_n$ be an $\eps$-corrupted set of samples from $\normal (0, I)$ of size $n$, where $n$ is as in Theorem \ref{thm:rspca1}.
Then, with probability $1 - \delta$, we have $\gamma \leq \rho / 2$.
\end{lemma}
\begin{proof}
Let $w$ be the uniform weights over the uncorrupted points.
Then it from Theorem \ref{thm:Xk-conc} that $\| \sum_{w} \sum_{i = 1}^n w_i (X_i X_i^T - I) \|^*_{\X_k} \leq O(\eta)$ with probability $1 - \delta$.
Since $w \in S_{n, \eps}$, this immediately implies that $\gamma \leq O(\rho)$.
By setting constants appropriately, we obtain the desired guarantee.
\end{proof}

We now show that if we are in Case 2, then $\gamma$ must be large:
\begin{lemma}
\label{lem:detect2}
Let $\rho, \delta > 0$.
Let $\eps, \eta, n$ be as in Theorem \ref{thm:rspca1}.
Let $X_1, \ldots, X_n$ be an $\eps$-corrupted set of samples from $\normal (0, I)$ of size $n$.
Then, with probability $1 - \delta$, we have $\gamma \geq (1 - \eps) \rho - (2 + \rho) \eta$.
In particular, for $\eps$ sufficiently small, and $\eta = O(\rho)$, we have that $\gamma > \rho / 2$.
\end{lemma}
\begin{proof}
Let $\Sigma = I + \rho vv^T$, and let $Y_i = \Sigma^{-1/2} X_i$, so that if $Y_i$ is uncorrupted, then $Y_i \sim \normal (0, I)$.
Let $w^*$ be the optimal solution to (\ref{eq:SDP-SPCA1}).
By Theorem \ref{thm:Xk-Sne-conc}, we have that with probability $1 - \delta$, we can write $\sum_{i = 1}^n w^*_i Y_i Y_i^T = w^g (I + N) + B$, where $\| N \|^*_{\X_k} \leq \eta$, and $B = \sum_{i \in \Sbad} w^*_i Y_i Y_i^T$.
Therefore, we have 
$
\sum_{i = 1}^n w^* X_i X_i^T = w^g (\Sigma + \Sigma^{1/2} N \Sigma^{1/2}) + \Sigma^{1/2} B \Sigma^{1/2} \; .
$
By definition, we have
\begin{align*}
\left\| \sum_{i = 1}^n w_i ( X_i X_i^T - I) \right\|^*_{\X_k} &\geq \langle w^g (\Sigma + \Sigma^{1/2} N \Sigma^{1/2}) + \Sigma^{1/2} B \Sigma^{1/2} - I, vv^T \rangle \\
&\geq w^g \langle (\Sigma + \Sigma^{1/2} N \Sigma^{1/2}), vv^T \rangle - 1 \\
&= w^g (1 + \rho) + w^g  v^T \Sigma^{1/2} N \Sigma^{1/2} v - 1 \\
&\geq (1 - \eps) \rho + (1 - \eps) v^T \Sigma^{1/2} N \Sigma^{1/2} v - \eps \; .
\end{align*}
It thus suffices to show that $| v^T \Sigma^{1/2} N \Sigma^{1/2} v | <(1 + \rho) \eta$.
Since $v$ is an eigenvector for $\Sigma$ with eigenvalue $1 + \rho$, we have that $\Sigma^{1/2} v = \sqrt{\rho + 1} \cdot v$ and thus
\begin{align*}
v^T \Sigma^{1/2} N \Sigma^{1/2} v &= (1 + \rho) v^T N v = (1 + \rho) \langle N, vv^T \rangle \leq (1 + \rho) \| N \|^*_{\X_k} \leq (1 + \rho) \eta \; .
\end{align*}
\end{proof}

Lemmas \ref{lem:detect1} and \ref{lem:detect2} together imply the correctness of \textsc{DetectRobustSPCA} and Theorem \ref{thm:rspca1}.

\section{An algorithm for robust sparse PCA recovery}

In this section, we prove Theorem \ref{thm:rspca2}.
We give some intuition here.
Perhaps the first naive try would be to simply run the same SDP in (\ref{eq:SDP-SPCA1}), and hope that the dual norm maximizer gives you enough information to recover the hidden spike.
This would more or less correspond to the simplest modification SDP of the sparse PCA in the non-robust setting that one could hope gives non-trivial information in this setting.
However, this cannot work, for the following straightforward reason: the value of the SDP is always at least $O(\rho)$, as we argued in Section \ref{sec:detection}.
Therefore, the noise can pretend to be some other sparse vector $u$ orthogonal to $v$, so that the covariance with noise looks like $w^g (I + \rho vv^T) + w^g \rho u u^T$, so that the value of the SDP can be minimized with the uniform set of weights.
Then it is easily verified that both $vv^T$ and $u u^T$ are dual norm maximizers, and so the dual norm maximizer does not uniquely determine $v$.

To circumvent this, we simply add an additional slack variable to the SDP, which is an additional matrix in $\X_k$, which we use to try to maximally explain away the rank-one part of $I + \rho vv^T$.
This forces the value of the SDP to be very small, which allows us to show that the slack variable actually captures $v$.

\subsection{The algorithm}
Our algorithms and analyses will make crucial use of the following convex set, which is a further relaxation of $\X_k$:
\begin{align*}
\Xt_k &= \left\{X \in \R^{d \times d}: \tr (X) \leq 2, \| X \|_2 \leq 1, \| X \|_1 \leq 3k, X \succeq 0 \right\} \; .
\end{align*}

Our algorithm, given formally in Algorithm \ref{alg:robustSPCA2}, will be the following.
We solve a convex program which simultaneously chooses a weights in $S_{n, \eps}$ and a matrix $A \in \W_k$ to minimize the $\W_k$ distance between the sample covariance with these weights, and $A$.
Our output is then just the top eigenvector of $A$.

\begin{algorithm}[htb]
\begin{algorithmic}[1]
\Function{RecoverRobustSPCA}{$X_1, \ldots, X_n, \eps, \delta, \rho$}
\State Let $w^*, A^*$ be the solution to
\begin{equation}
\label{eq:SDP-SPCA}
\argmin_{w \in S_{n, \eps}, A \in \X_k} \left\| \sum_{i = 1}^n w_i ( X_i X_i^T - I) - \rho A \right\|^*_{\W_{2k}}
\end{equation}
\State Let $u$ be the top eigenector of $A^*$
\State \textbf{return} The $d_{\U_k} (u) \| u \|^*_{\U_k}$, i.e., the vector with all but the top $k$ coordinates of $v$ zeroed out.
\EndFunction
\end{algorithmic}
\caption{Learning a spiked covariance model, robustly}
\label{alg:robustSPCA2}
\end{algorithm}

\noindent This algorithm can be run efficiently for the same reasons as explained for \textsc{DetectRobustSPCA}.
For the rest of the section we will assume that we have an exact solution for this problem.
As before, we only use information about $A$, and since $A$ comes from a bounded space, and our analysis is robust to small perturbations in $A$, this does not change anything.

\subsection{More concentration bounds}
Before we can prove correctness of our algorithm, we require a couple of concentration inequalities for the set $\W_k$.
\begin{lemma}
\label{lem:Wk-PCA}
Fix $\eps, \delta > 0$.
Let $X_1, \ldots, X_n \sim \normal (0, I)$, where $n$ is as in Theorem \ref{thm:Xk-conc}.
Then with probability $1 - \delta$
\[
\left\| \frac{1}{n} \sum_{i = 1}^n X_i X_i^T - I \right\|^*_{\W_k} \leq O(\eps) \; .
\]
\end{lemma}
\begin{proof}
Let $\Sigmahat$ denote the empirical covariance.
Observe that $\W_k \subseteq \bigcup_{i = 0}^\infty 2^{-i} \X_{2^{i + 1} k}$.
Moreover, for any $i$, by Theorem \ref{thm:Xk-conc}, if we take
\begin{align*}
n &= \Omega\left( \frac{\min(d, (2^{i + 1} k)^2) + \log \binom{d^2}{(2^{i + 1} k)^2} + \log 1 / \delta}{(2^{-i} \eps)^2} \right) \\
&= \Omega \left( \frac{\min (d, k^2) + \log \binom{d^2}{k^2} + 2^{2i} \log 1/\delta}{\eps^2}\right) \; ,
\end{align*}
then $|\langle M, \Sigmahat \rangle| \leq \eps$ for all $M \in 2^{-i} \X_{2^{i + 1} k}$ with probability $1 - \delta / 2$.
In particular, if we take
\[
n = \Omega \left( \frac{\min (d, k^2) + \log \binom{d^2}{k^2} + \log 1/\delta}{\eps^2}\right)
\]
samples, then for any $i$, we have $|\langle M, \Sigmahat \rangle| \leq \eps$ for all $M \in 2^{-1} \X_{2^{i + 1} k}$ with probability at least $1 - \delta^{2^{2i}} / 2$.
By a union bound over all these events, since $\sum_{i = 0}^\infty \delta^{2^{2i}} \leq 2 \delta$, we conclude that if we take $n$ to be as above, then $|\langle M, \Sigmahat \rangle| \leq \eps$ for all $M \in \bigcup_{i = 0}^\infty 2^{-i} \X_{2^{i + 1} k}$ with probability $1 - \delta$.
Since $\W_k$ is contained in this set, this implies that $\| \Sigmahat - \Sigma \|^*_{\W_k} \leq O(\eps)$ with probability at least $1 - \delta$, as claimed.
\end{proof}

By the same techniques as in the proofs of Theorems \ref{thm:Uk-Sne-conc} and \ref{thm:Xk-Sne-conc}, we can show the following bound.
Because of this, we omit the proof for conciseness.
\begin{corollary}
\label{cor:Xk-Sne-PCA}
Fix $\eps, \delta > 0$.
Let $X_1, \ldots, X_n \sim \normal (0, I)$ where $n$ is as in Theorem \ref{thm:Xk-Sne-conc}.
Then there is an $\eta = O(\eps \sqrt{\log 1 / \eps})$ so that
\begin{align*}
\Pr \left[ \exists w \in S_{n, \eps} : \left\| \sum_{i = 1}^n w_i X_i X_i^T - I \right\|^*_{\W_k} \geq \eta \right] &\leq \delta \; .
\end{align*}
\end{corollary}

\subsection{Proof of Theorem \ref{thm:rspca2}}
In the rest of this section we will condition on the following deterministic event happening:
\begin{equation}
\label{eq:rspca}
\forall w \in S_{n, \eps} : \left\| \sum_{i = 1}^n w_i X_i X_i^T - I \right\|^*_{\W_{2k}} \leq \eta \; ,
\end{equation}
where $\eta = O(\eps \log 1 / \eps)$.
By Corollary \ref{cor:Xk-Sne-PCA}, this holds if we take 
\[
n = \Omega \left( \frac{\min(d, k^2) + \log \binom{d^2}{k^2} + \log 1 / \delta} {\eta_2^2}\right) 
\]
samples.

The rest of this section is dedicated to the proof of the following theorem, which immediately implies Theorem \ref{thm:rspca2}.
\begin{theorem}
\label{thm:rspca3}
Fix $\eps, \delta,$ and let $\eta$ be as in (\ref{eq:rspca}).
Assume that (\ref{eq:rspca}) holds.
\newline
Let $\widehat{v}$ be the output of $\textsc{RecoveryRobustSPCA} (X_1, \ldots, X_n, \eps, \delta, \rho)$.
Then $L(\widehat{v}, v) \leq O(\sqrt{(1 + \rho) \eta / \rho})$.
\end{theorem}

Our proof proceeds in a couple of steps.
Let $\Sigma = I + \rho vv^T$ denote the true covariance.
We first need the following, technical lemma:
\begin{lemma}
\label{lem:1plusrho}
Let $M \in \W_k$. Then $\Sigma^{1/2} M \Sigma^{1/2} \in (1 + \rho) \W_k$.
\end{lemma}
\begin{proof}
Clearly, $\Sigma^{1/2} M \Sigma^{1/2} \succeq 0$.
Moreover, since $\Sigma^{1/2} = I + (\sqrt{1 + \rho} - 1) vv^T$, we have that the maximum value of any element of $\Sigma^{1/2}$ is upper bounded by $\sqrt{1 + \rho}$.
Thus, we have $\| \Sigma^{1/2} M \Sigma^{1/2} \|_1 \leq (1 + \rho) \| M \|_1$.
We also have 
\begin{align*}
\tr (\Sigma^{1/2} M \Sigma^{1/2}) &= \tr(\Sigma M) \\
&= \tr(M) + \rho v^T M v \leq 1 + \rho \; ,
\end{align*} 
since $\| M \| \leq 1$.
Thus $\Sigma^{1/2} M \Sigma^{1/2} \in (1 + \rho) \W_k$, as claimed.
\end{proof}

Let $w^*, A^*$ be the output of our algorithm.
We first claim that the value of the optimal solution is quite small:
\begin{lemma}
\label{lem:SPCA1}
\[
\left\| \sum_{i = 1}^n w^*_i (X_i X_i^T - I) - \rho A^* \right\|^*_{\W_{2k}} \leq \eta (1 + \rho) \; .
\]
\end{lemma}
\begin{proof}
Indeed, if we let $w$ be the uniform set of weights over the good points, and we let $A = vv^T$, then by (\ref{eq:rspca}), we have
\[
\sum_{i = 1}^n w_i X_i X_i^T = \Sigma^{1/2} (I + N) \Sigma^{1/2} \; ,
\]
where $\| N \|^*_{\X_k} \leq \eta$, and $\Sigma = I + \rho vv^T$.
Thus we have that 
\begin{align*}
\left\| \sum_{i = 1}^n w_i (X_i X_i^T - I) - \rho vv^T \right\|^*_{\W_{2k}} &= \| \Sigma^{1/2} N \Sigma^{1/2} \|^*_{\W_{2k}} \\
&= \max_{M \in \W_k} \left| \tr (\Sigma^{1/2} N \Sigma^{1/2} M) \right| \\
&=  \max_{M \in \W_k} \left| \tr (N \Sigma^{1/2} M \Sigma^{1/2}) \right| \\
&\leq (1 + \rho) \| N \|^*_{\W_{2k}} \; ,
\end{align*}
by Lemma \ref{lem:1plusrho}.
\end{proof}

We now show that this implies the following:
\begin{lemma}
\label{lem:SPCA2}
$v^T A^* v \geq 1 - (2 + 3 \rho) \eta / \rho$.
\end{lemma}
\begin{proof}
By (\ref{eq:rspca}), we know that we may write $\sum_{i = 1}^n w_i (X_i X_i^T - I) = w^g \rho vv^T + B - (1 - w^g) I + N$, where $B = \sum_{i \in \Sbad} w_i X_i X_i^T$, and $\| N \|^*_{\W_k} \leq (1 + \rho) \eta$.
Thus, by Lemma \ref{lem:SPCA1} and the triangle inequality, we have that
\begin{align*}
\left\| w^g \rho vv^T + B - \rho A \right\|^*_{\W_k} & \leq \eta + \| N \|^*_{\W_k} + (1 - w^g) \| I \|^*_{\W_k} + (1 - w^g) \| \rho A \|^*_{\W_k} \\
&\leq (1 + \rho) \eta + \eps + \rho \eps \\
&\leq (1 + 2 \rho) \eta + \eps \; .
\end{align*}
Now, since $vv^T \in \W_k$, the above implies that
\[
| w^g \rho + v^T B v - \rho v^T A^* v | \leq (1 + 2 \rho) \eta + \eps \; ,
\]
which by a further triangle inequality implies that
\[
| \rho (1 - v^T A^* v) + v^T B v | \leq(1 + 2 \rho) \eta + \eps + \eps \rho \leq (2 + 3 \rho) \eta \; .
\]
Since $0 \leq v^T A^* v \leq 1$ (since $A \in \X_k$) and $B$ is PSD, this implies that in fact, we have
\[
0 \leq \rho (1 - v^T A^* v) \leq (2 + 3 \rho) \eta \; .
\]
Hence $v^T A^* v \geq 1 - (2 + 3 \rho) \eta / \rho$, as claimed.
\end{proof}

Let $\gamma = (2 + 3 \rho) \eta / \rho$.
The lemma implies that the top eigenvalue of $A^*$ is at least $1 - \gamma$.
Moreover, since $A^* \in \X_k$, as long as $\gamma \leq 1/2$, this implies that the top eigenvector of $A^*$ is unique up to sign.
By the constraint that $\eta \leq O(\min(\rho, 1))$, for an appropriate choice of constants, we that $\gamma \leq 1 / 10$, and so this condition is satisfied.
Recall that $u$ is the top eigenvector of $A^*$.
Since $\tr (A^*) = 1$ and $A^*$ is PSD, we may write $A^* = \lambda_1 u u^T + A_1$, where $u$ is the top eigenvector of $A^*$, $\lambda_1 \geq 1 - \gamma$, and $\| A_1 \| \leq \gamma$.
Thus, by the triangle inequality, this implies that
\[
\| \rho(vv^T - \lambda_1 uu^T) + B \|^*_{\X_{2k}} \leq O(\rho \gamma)
\]
which by a further triangle inequality implies that
\begin{equation}
\label{eq:B-bound}
\| \rho(vv^T - uu^T) + B \|^*_{\X_{2k}} \leq O(\rho \gamma) \; .
\end{equation}

We now show this implies the following intermediate result:
\begin{lemma}
$(v^T u)^2 \geq 1 - O(\gamma)$.
\end{lemma}
\begin{proof}
By Lemma \ref{lem:SPCA2}, we have that $v^T A^* v = \lambda_1 (v^T u)^2 + v^T A_1 v \geq 1 - \gamma$.
In particular, this implies that $(v^T u)^2 \geq (1 - 2 \gamma) / \lambda_1 \geq 1 - 3 \gamma$, since $1 - \gamma \leq \lambda \leq 1$.
\end{proof}

We now wish to control the spectrum of $B$.
For any subsets $S, T \subseteq [d]$, and for any vector $x$ and any matrix $M$, let $x_S$ denote $x$ restricted to $S$ and $M_{S, T}$ denote the matrix restricted to the rows in $S$ and the columns in $T$.
Let $I$ be the support of $u$, and let $J$ be the support of the largest $k$ elements of $v$.
\begin{lemma}
\label{lem:BI-bound}
$\| B_{I, I} \| \leq O(\rho \gamma)$.
\end{lemma}
\begin{proof}
Observe that the condition (\ref{eq:B-bound}) immediately implies that 
\begin{equation}
\label{eq:IBound1}
\| \rho(v_I v_I^T - u_I u_I^T) + B_{I, I} \| \leq c \rho \gamma \; ,
\end{equation}
for some $c$, since any unit vector $x$ supported on $I$ satisfies $xx^T \in \X_{2k}$.
Suppose that $\| B_{I, I} \| \geq C \gamma$ for some sufficiently large $C$.
Then (\ref{eq:IBound1}) immediately implies that $\| \rho (v_I v_I^T - u_I u_I^T) \| \geq (C - c) \rho \gamma$.
Since $(v_I v_I^T - u_I u_I^T)$ is clearly rank 2, and satisfies $\tr (v_I v_I^T - u_I u_I^T) = 1 - \| u_I \|_2^2 \geq 0$, this implies that the largest eigenvalue of $v_I v_I^T - u_I u_I^T$ is positive.
Let $x$ be the top eigenvector of $v_I v_I^T - u_I u_I^T$.
Then, we have $x^T (v_I v_I^T - u_I u_I^T) x + x^T B x = (C - c) \rho \gamma + x^T B x \geq (C - c) \rho \gamma$ by the PSD-ness of $B$.
If $C > c,$ this contradicts (\ref{eq:IBound1}), which proves the theorem.
\end{proof}
This implies the following corollary:
\begin{corollary}
\label{cor:vI-Bound}
$\| u_I \|_2^2 \geq 1 - O(\gamma)$.
\end{corollary}
\begin{proof}
Lemma \ref{lem:BI-bound} and (\ref{eq:IBound1}) together imply that $\| v_I v_I^T - u_I u_I^T \| \leq O(\gamma)$.
The desired bound then follows from a reverse triangle inequality.
\end{proof}

We now show this implies a bound on $B_{J \setminus I, J \setminus I}$:
\begin{lemma}
\label{lem:BJ-bound}
$\| B_{J \setminus I, J \setminus I} \| \leq O(\rho \gamma)$.
\end{lemma}
\begin{proof}
Suppose $\| B_{J \setminus I, J \setminus I} \| \geq C \gamma$ for some sufficiently large $C$.
Since $u$ is zero on $J \setminus I$, (\ref{eq:B-bound}) implies that
\[
\|\rho v_{J \setminus I} v_{J \setminus I}^T + B_{J \setminus I, J \setminus I} \| \leq c \rho \gamma \; ,
\]
for some universal $c$.
By a triangle inequality, this implies that $\| v_{J \setminus I} \|_2^2 = \| v_{J \setminus I} v_{J \setminus I}^T \| \geq (C - c) \gamma$.
Since $v$ is a unit vector, this implies that $\| v_I \|_2^2 \leq 1 - (C - c) \gamma$, which for a sufficiently large $C$, contradicts Corollary \ref{cor:vI-Bound}.
\end{proof}
We now invoke the following general fact about PSD matrices:
\begin{lemma}
\label{lem:PSD-middle}
Suppose $M$ is a PSD matrix, written in block form as
\[
M = \left( \begin{array}{cc}
C & D \\
D^T & E \end{array} \right) \; .
\]
Suppose furthermore that $\| C \| \leq \xi$ and $\| E \| \leq \xi$.
Then $\| M \| \leq O(\xi)$.
\end{lemma}
\begin{proof}
It is easy to see that $\| M \| \leq O(\max (\| C \|, \| D \|, \| E \|))$.
Thus it suffices to bound the largest singular value of $D$.
For any vectors $\phi, \psi$ with appropriate dimension, we have that 
\begin{align*}
(\phi^T~  -\psi^T)~ M \left( \begin{array}{c}
\phi \\
-\psi \end{array} \right) &=  \phi^T A \phi - 2 \phi^T D \psi + \psi^T C \psi \geq 0 \; ,
\end{align*}
which immediately implies that the largest singular value of $D$ is at most $(\| A \| + \| B \|) / 2$, which implies the claim.
\end{proof}
Therefore, Lemmas \ref{lem:BI-bound} and \ref{lem:BJ-bound} together imply:
\begin{corollary}
\label{cor:IJ-bound}
$\| v_{I \cup J} v_{I \cup J}^T - u_{I \cup J} u_{I \cup J}^T \| \leq O( \gamma) \; .$
\end{corollary}
\begin{proof}
Observe (\ref{eq:B-bound}) immediately implies that $\| \rho(v_{I \cup J} v_{I \cup J}^T - u_{I \cup J} u_{I \cup J}^T) + B_{I \cup J, I \cup J}\| \leq O(\rho \gamma)$, since $|I \cup J| \leq 2 k$.
Moreover, Lemmas \ref{lem:BI-bound} and \ref{lem:BJ-bound} with Lemma \ref{lem:PSD-middle} imply that $\| B_{I \cup J, I \cup J} \| \leq O(\rho \gamma)$, which immediately implies the statement by a triangle inequality.
\end{proof}

Finally, we show this implies $\| v v^T  - u_J u_J^T \| \leq O(\gamma)$, which is equivalent to the theorem.
\begin{proof}[Proof of Theorem \ref{thm:rspca3}]
We will in fact show the slightly stronger statement, that $\| u u^T - v_J v_J^T \|_F \leq O(\gamma)$.
Observe that since $uu^T - vv^T$ is rank 2, Corollary \ref{cor:IJ-bound} implies that $\| v_{I \cup J} v_{I \cup J}^T - u_{I \cup J} u_{I \cup J}^T \|_F \leq O(\gamma)$, since for rank two matrices, the spectral and Frobenius norm are off by a constant factor.
We have 
\[
\| uu^T - vv^T \|_F^2 = \sum_{(i, j) \in I \cap J \times I \cap J} (u_i u_j - v_i v_j)^2 +  \sum_{(i, j) \in I \times I \setminus J \times J} (v_i v_j)^2 + \sum_{(i, j) \in J \times J \setminus I \times I} (u_i u_j)^2 \; .
\]
We have 
\[
\sum_{(i, j) \in I \cap J \times I \cap J} (u_i u_j - v_i v_j)^2 + \sum_{(i, j) \in J \times J \setminus I \times I} (u_i u_j)^2 \leq \| v_{I \cup J} v_{I \cup J}^T - u_{I \cup J} u_{I \cup J}^T \|^2 \leq O(\gamma) \; ,
\]
by Corollary \ref{cor:IJ-bound}.
Moreover, we have that 
\begin{align*}
\sum_{(i, j) \in I \times I \setminus J \times J} (v_i v_j)^2 &\leq 2 \left( \sum_{(i, j) \in I \times I \setminus J \times J} (v_i v_j - u_i u_j)^2 + \sum_{(i, j) \in I \times I \setminus J \times J} (u_i u_j)^2 \right) \\
&\leq 2 \left( \| v_{I \cup J} v_{I \cup J}^T - u_{I \cup J} u_{I \cup J}^T \|^2 + \sum_{(i, j) \in I \times I \setminus J \times J} (u_i u_j)^2 \right) \\
&\leq 2 \left( \| v_{I \cup J} v_{I \cup J}^T - u_{I \cup J} u_{I \cup J}^T \|^2 + \sum_{(i, j) \in J \times J \setminus I \times I} (u_i u_j)^2 \right) \\
&\leq O(\gamma) \; .
\end{align*}
since $J \times J$ contains the $k^2$ largest entries of $u u^T$.
This completes the proof.
\end{proof}

\section*{Acknowledgements}
The author would like to thank Ankur Moitra for helpful advice throughout the project, and Michael Cohen for some surprisingly\footnote{Is it really surprising though?} useful conversations.

\clearpage

\bibliographystyle{alpha}

\bibliography{refs}

\appendix

\section{Information theoretic estimators for robust sparse estimation}
\label{app:info}
This section is dedicated to the proofs of the following two facts:
\begin{fact}
\label{fact:spMeanInef}
Fix $\eps, \delta > 0$, and let $k$ be fixed.
Given an $\eps$-corrupted set of samples $X_1, \ldots, X_n \in \R^d$ from $\normal (\mu, I)$, where $\mu$ is $k$-sparse, and 
\[
n = O \left( \frac{ k \log (d  / \eps) + \log 1 / \delta }{\eps^2} \right) \; ,
\]
there is an (inefficient) algorithm which outputs $\muhat$ so that with probability $1 - \delta$, we have $\| \mu - \muhat \|_2 \leq O(\eps)$.
Moreover, up to logarithmic factors, this rate is optimal.
\end{fact}

\begin{fact}
\label{fact:spcaInef}
Fix $\rho, \delta > 0$.
Suppose that $\rho = O(1)$.
 Then, there exist universal constants $c, C$ so that: (a) if $\eps \leq c \rho$, and we are given a $\eps$-corrupted set of samples from either $\normal(0, I)$ or $\normal (0, I + \rho vv^T)$ for some $k$-sparse unit vector $v$ of size
\[
n = \Omega \left( \frac{k + \log \binom{d}{k} + \log 1 / \delta}{\rho^2} \right) \; ,
\]
then there is an (inefficient) algorithm which succeeds with probability $1 - \delta$ for the detection problem.
Moreover, if $\eps \geq C \rho$, then no algorithm succeeds with probability greater than $1/2$, and this statistical rate is optimal.
\end{fact}

The rates in Facts \ref{fact:spMeanInef} and \ref{fact:spcaInef} are already known to be optimal (up to log factors) without noise.
Thus in this section we focus on proving the upper bounds, and the lower bounds on error.

The lower bounds on what error is achievable follow from the following two facts, which follow from Pinsker's inequality (see e.g. \cite{CT12}), and the fact that the corruption model can, given samples from $D_1$, simulate samples from $D_2$ by corrupting an $O(\eps)$ fraction of points, if $\dtv (D_1, D_2) \leq O(\eps)$.
\begin{fact}
\label{fact:pinsker}
Fix $\eps > 0$ sufficiently small.
Let $\mu_1, \mu_2$ be arbitrary.
There is some universal constant $C$ so that if $\dtv (\normal (\mu_1, I), \mu_2, I) \leq \eps$, then $\| \mu_1 - \mu_2 \|_2 \leq C \eps$, and if $\| \mu_1 - \mu_2 \|_2 \leq \eps$, then $\dtv (\normal (\mu_1, I), \normal(\mu_2, I)) \leq C \eps$.
\end{fact}

\begin{fact}
\label{fact:pinsker1}
Fix $\rho = O(1)$.
Let $u, v$ be arbitrary unit vectors.
Then $\dtv( \normal (0, I), \normal (0, I + \rho vv^T)) = \Theta (\rho)$, and $\dtv (\normal (0, I + \rho vv^T), \normal (0, I + \rho uu^T)) = O(L(u, v))$.
\end{fact}

Our techniques for proving the upper bounds go through the technique of agnostic hypothesis selection via tournaments.
Specifically, we use the following lemma:
\begin{lemma}[\cite{DKKLMS16}, Lemma 2.9]
\label{tournamentLem}
Let $\mathcal{C}$ be a class of probability distributions. Suppose that for some $n,\eps,\delta>0$ there exists an algorithm that given an $\eps$-corrupted set of samples from some $D\in\mathcal{C}$, returns a list of $M$ distributions so that with $1-\delta/3$ probability there exists a $D'\in M$ with $\dtv(D',D)<\gamma$. Suppose furthermore that with probability $1 - \delta / 3$, the distributions returned by this algorithm are all in some fixed set $\mathcal{M}$. Then there exists another algorithm, which given $O(N+(\log(|\mathcal{M}|)+\log(1/\delta))/\epsilon^2)$ samples from $\Pi$, an $\eps$-fraction of which have been arbitrarily corrupted, returns a single distribution $\Pi'$ so that with $1-\delta$ probability $\dtv(D',D)<O(\gamma+\eps)$.
\end{lemma}

\subsection{Proof of Upper Bound in Fact \ref{fact:spMeanInef}}
Let $\mathcal{M}_A$ be the set of distributions $\{\normal (\mu', I)\}$, where $\mu'$ ranges over the set of $k$-sparse vectors so that each coordinate of $\mu'$ is an integer multiple of $\eps / (10 \sqrt{d})$, and so that $\| \mu' - \mu \|_2 \leq A$.
We then have:

\begin{claim}
\label{claim:meanNet}
There exists a $\normal (\mu', I) = D \in \mathcal{M}_A$ so that $\| \mu - \mu' \|_2 \leq O(\eps)$.
Moreover, $|\mathcal{M}_A| \leq \binom{d}{k} \cdot (10 A \sqrt{d} / \eps)^k$.
\end{claim}
\begin{proof}
The first claim is straightforward.
We now prove the second claim.
For each possible set of $k$ coordinates, there are at most $(10 A \sqrt{d} / \eps)^k$ vectors supported on those $k$ coordinates with each coordinate being an integer multiple of $\eps / (10 \sqrt{d})$ with distance at most $A$ from any fixed vector.
Enumerating over all $\binom{d}{k}$ possible choices of $k$ coordinates yields the desired answer.
\end{proof}

The estimator is given as follows: first, run $\textsc{NaivePrune} (X_1, \ldots, X_n, \delta)$ to output some $\mu_0$ so that with probability $1 - \delta$, we have $\| \mu_0 - \mu \|_2 \leq O(\sqrt{d \log n / \delta})$.
Round each coordinate of $\mu_0$ so that it is an integer multiple of $\eps / (10 \sqrt{d})$.
Then, output the set of distributions $\mathcal{M'} = \{ \normal (\mu'', I) \}$, where $\mu''$ is any $k$-sparse vector with each coordinate being an integer multiple of $\eps / (10 \sqrt{d})$, with $\| \alpha \|_2 \leq O(\sqrt{d \log n / \delta})$.
With probability $1 - \delta$, we have $\mathcal{M'} \subseteq \mathcal{M}_{O(\sqrt{d \log n / \delta})}$.
By Claim \ref{claim:meanNet}, applying Lemma \ref{tournamentLem} to this set of distributions yields that we will select, with probability $1 - \delta$, a $\mu'$ so that $\| \mu - \mu' \|_2 \leq O(\eps)$.
By Claim \ref{claim:meanNet}, this requires
\[
O\left( \frac{\log |\mathcal{M}_{O(\sqrt{d \log n / \delta})}|}{\eps^2} \right) = O \left( \frac{ \log \binom{d}{k} + k \log (d  / \eps) + \log 1 / \delta }{\eps^2} \right) \; ,
\]
samples, which simplifies to the desired bound, as claimed.

\subsection{Proof of Upper Bound in Fact \ref{fact:spcaInef}}
Our detection algorithm is given as follows.
We let $\mathcal{N}$ be an $O(1)$-net over all $k$-sparse unit vectors, and we apply Lemma \ref{tournamentLem} to the set $\{\normal (0, I + \rho u u^T)\}_{u \in \mathcal{N}} \cup \{ \normal (0, I)\}$.
Clearly, we have:
\begin{claim}
$|\mathcal{M}| = \binom{d}{k} 2^{O(k)}$.
\end{claim}
By Fact \ref{fact:pinsker1} and the guarantees of Lemma \ref{tournamentLem}, by an appropriate setting of parameters, if we have
\[
n = O \left( \frac{\log |\mathcal{M}| + \log 1 / \delta}{\eps^2} \right) = O \left( \frac{k + \log \binom{d}{k} + \log 1 / \delta}{\eps^2} \right)
\]
samples, then with probability $1 - \delta$ we will output $\normal (0, I)$ if and only if the true model is $\normal (0, I)$.
This proves the upper bound.

\section{Omitted Details from Section \ref{sec:dual}}
\subsection{Writing non-robust algorithms as dual norm maximization}
\label{app:dualnorm}
In this section we will briefly review well-known non-robust algorithms for sparse mean recovery and for sparse PCA, and write them using our language.
\paragraph{Thresholding}
Recall that in the (non-robust) sparse mean estimation problem, one is given samples $X_1, \ldots, X_n \sim \normal (\mu, I)$ where $\mu$ is $k$-sparse. The goal is then to recover $\mu$.
It turns out the simple thresholding algorithm \textsc{ThresholdMean} given in Algorithm \ref{alg:thresh} suffices for recovery:

\begin{algorithm}[htb]
\begin{algorithmic}[1]
\Function{ThresholdMean}{$X_1, \ldots, X_n$}
\State Let $\muhat = \frac{1}{n} \sum_{i = 1}^n X_i$
\State Let $S$ be the set of $k$ coordinates of $\muhat$ with largest magnitude
\State Let $\muhat'$ be defined to be $\muhat'_i = \muhat_i$ if $i \in S$, $0$ otherwise
\State {\bf return} $\muhat'$
\EndFunction
\end{algorithmic}
\caption{Thresholding for sparse mean estimation}
\label{alg:thresh}
\end{algorithm}
\noindent The correctness of this algorithm follows from the following folklore result, whose proof we shall omit for conciseness:
\begin{fact}[c.f. \cite{Rig15}]
Fix $\eps, \delta > 0$, and let $X_1, \ldots, X_n$ be samples from $\normal (\mu, I)$, where $\mu$ is $k$-sparse and
\[
n = \Omega \left( \frac{\log \binom{d}{k} + \log 1 / \delta}{\eps^2} \right) \; .
\]
Then, with probability $1 - \delta$, if $\muhat'$ is the output of \textsc{ThresholdMean}, we have $\| \muhat' - \muhat \|_2 \leq \eps$.
\end{fact}

To write this in our language, observe that
\[
\textsc{ThresholdSMean}(X_1, \ldots, X_n) = \| \muhat \|_{\U_k}^* \cdot d_{\U_k} (\muhat) \; ,\]
 where $\muhat = \frac{1}{n} \sum_{i = 1}^n X_i$.

\paragraph{$L_1$ relaxation}
In various scenarios, including recovery of a spiked covariance, one may envision the need to take $k$-sparse eigenvalues a matrix $A$, that is, vectors which solve the following non-convex optimization problem:
\begin{align*}
\mbox{max}&~ v^T A v\\
\mbox{s.t.}&~\| v \|_2 = 1,~\| v \|_0 \leq k\; . \numberthis \label{eq:l0-norm}
\end{align*}
However, this problem is non-convex and cannot by solved efficiently.
This motivates the following SDP relaxation of (\ref{eq:l0-norm}):
First, one rewrites the problem as 
\begin{align*}
\mbox{max}&~ \tr (A X )\\
\mbox{s.t.}&~\tr (X) = 1,~\| X \|_0 \leq k^2\;, X \succeq 0\;, \mathrm{rank}(X) = 1 \;  \numberthis \label{eq:l0-norm1}
\end{align*}
where $\| X \|_0$ is the number of non-zeros of $X$.
Observe that since $X$ is rank $1$ if we let $X = v v^T$ these two problems are indeed equivalent.
Then to form the SDP, one removes the rank constraint, and relaxes the $\ell_0$ constraint to a $\ell_1$ constraint:
\begin{align*}
\mbox{max}&~ \tr (A X )\\
\mbox{s.t.}&~\tr (X) = 1,~\| X \|_1 \leq k\;, X \succeq 0\;.  \numberthis \label{eq:l0-norm2}
\end{align*}
The work of \cite{DE-GJL07} shows that this indeed detects the presence of a spike (but at an information theoretically suboptimal rate).

Finally, by definition, for any PSD matrix $A$, if $X$ is the solution to (\ref{eq:l0-norm2}) with input $A$, we have $X = d_{\X_k} (A)$.

\subsection{Numerical precision} 
\label{sec:numerical}

In general, we cannot find closed form solutions for $d_{\X_k} (A)$ in finite time.
However, it is well-known that we can find these to very high numerical precision in polynomial time.
For instance, using the ellipsoid method, we can find an $M'$ so that $\| M' - d_{\X_k} (A) \|_\infty \leq \eps$ in time $\poly (d, \log 1 / \eps)$.
It is readily verified that if we set $\eps' = \poly (\eps, 1 / d)$ then the numerical precision of the answer will not effect any of the calculations we make further on.
Thus for simplicity of exposition we will assume throughout the paper that given any $A$, we can find  $d_{\X_k} (A)$ exactly in polynomial time.

\section{Omitted Proofs from Section \ref{sec:dual}}
\label{app:proofs-dual}
\begin{proof}[Proof of Theorem \ref{thm:Uk-Sne-conc}]
Fix $n$ as in Theorem \ref{thm:Uk-Sne-conc}, and let $\delta_1 = \binom{n}{\eps n}^{-1} \delta$.
By convexity of $S_{n, \eps}$ and the objective function, it suffices to show that with probability $1 - \delta$, the following holds:
\[
\forall w_I ~\mbox{s.t.}~|I| = (1 - \eps) n: \left\|  \sum_{i = 1}^n w_i X_i \right\|^*_{\U_k} \leq \eta \; .
\]

Condition on the event that 
\begin{equation}
\label{eq:unif-bound}
\left\| \frac{1}{n} \sum_{i = 1}^n X_i \right\|^*_{\U_k} \leq \eps \; .
\end{equation}
By Corollary \ref{cor:Uk-conc}, this occurs with probability $1 - O(\delta)$.

Fix any $I \subseteq [n]$ so that $|I| = (1 - \eps) n$. 
By Corollary \ref{cor:Uk-conc} applied to $I^c$, we have that there is some universal constant $C$ so that as long as 
\begin{equation}
\label{eq:n-cond}
\eps n \geq C \cdot \frac{ \min (d, k^2) + \log \binom{d^2}{k^2} + \log \binom{n}{\eps n} + \log 1 / \delta }{\alpha^2} \; ,
\end{equation}
then with probability $1 - \delta'$,
\begin{equation}
\label{eq:w-cond}
 \left\| \frac{1}{\eps n} \sum_{i \not\in I} X_i \right\|^*_{\U_k} \leq \alpha \; .
\end{equation}
Since $\log \binom{n}{\eps n} = \Theta (n \eps \log 1 / \eps)$, (\ref{eq:n-cond}) is equivalent to the condition that 
\[
n \left( \eps - C \frac{\eps \log 1 / \eps}{\alpha^2} \right) \geq C \cdot \frac{\min (d, k^2) + \log \binom{d^2}{k^2} + \log 1 / \delta}{\alpha^2} \; .
\]
Let $\alpha = O(\sqrt{\log 1/ \eps})$.
By our choice of $\eta$, we have that $0 \leq \eps - \frac{\eps \log 1 / \eps}{\eta^2} \leq \eps /(2C)$, and by an appropriate setting of constants, since by our choice of $n$ we have
\[
\frac{\eps n}{2} \geq C \cdot \frac{\min (d, k^2) + \log \binom{d^2}{k^2} + \log 1 / \delta}{\alpha^2} \; ,
\]
we have that (\ref{eq:w-cond}) holds with probability $1 - \delta'$.
Thus by a union bound over all $\binom{n}{\eps n}$ choices of $I$ so that $|I| = (1 - \eps) n$, we have that except with probability $1 - \delta$, we have that (\ref{eq:w-cond}) holds simultaneously for all $I$ with $|I| = (1 - \eps) n$.
The desire result then follows from this and (\ref{eq:unif-bound}), and a union bound.
\end{proof}

\begin{proof}[Proof of Theorem \ref{thm:Xk-Sne-conc}]
This follows from the exact same techniques as the proof of Theorem \ref{thm:Uk-Sne-conc}, by replacing all $\U_k$ with $\X_k$, and using Theorem \ref{thm:Xk-conc} instead of Corollary \ref{cor:Uk-conc}.

\end{proof}

\section{Computational Barriers for sample optimal robust sparse mean estimation}
\label{app:barrier}
We conjecture that the rate achieved by Theorem \ref{thm:robust-GSM} is tight for computationally efficient algorithms (up to log factors).
Intuitively, the major difficulty is that distinguishing between $\normal (\mu_1, I)$ and $\normal (\mu_2, I)$ given corrupted samples seems to inherently require second moment (or higher) information, for any $\mu_1, \mu_2 \in \R^d$.
Certainly first moment information by itself is insufficient.
In this sparse setting, this is very problematic, as this inherently asks for us to detect a large sparse eigenvector of the empirical covariance.
This more or less reduces to the problem solved by (\ref{eq:l0-norm}).
This in turn requires us to relax to the problem solved by SDPs for sparse PCA, for which we know $\Omega(k^2 \log d / \eps^2 )$ samples are necessary for non-trivial behavior to emerge.
We leave resolving this gap as an interesting open problem.

\end{document}